\documentclass{article}

 \usepackage[nonatbib,final]{neurips_2021}
\usepackage{comment}
\usepackage[utf8]{inputenc} 
\usepackage[T1]{fontenc}    
\usepackage{hyperref}       
\usepackage{cite}
\usepackage{url}            
\usepackage{booktabs}       
\usepackage{amsfonts}       
\usepackage{nicefrac}       
\usepackage{microtype}      
\usepackage{xcolor}         
\usepackage{graphicx}
\usepackage{wrapfig}
\usepackage[linesnumbered,ruled]{algorithm2e}
\usepackage{setspace}
\usepackage{amsthm}
\usepackage{amsmath}
\usepackage{amssymb}
\usepackage{enumitem}

\usepackage{amsmath,amsfonts,bm}

\def\eqref#1{equation~\ref{#1}}

\def\1{\bm{1}}

\DeclareMathAlphabet{\mathsfit}{\encodingdefault}{\sfdefault}{m}{sl}
\SetMathAlphabet{\mathsfit}{bold}{\encodingdefault}{\sfdefault}{bx}{n}

\DeclareMathOperator*{\argmax}{arg\,max}
\DeclareMathOperator*{\argmin}{arg\,min}

\usepackage{caption}
\usepackage{subcaption}
\usepackage{wrapfig}
\usepackage{array,multirow,graphicx}
\newtheorem{theorem}{Theorem}

\newtheorem{definition}{Definition}

\title{Adversarial Graph Augmentation 
to Improve \\ Graph Contrastive Learning}

\author{%
    Susheel Suresh \\
    Purdue University\\
    \texttt{suresh43@purdue.edu} \\
\And
    Pan Li\thanks{Pan Li and Jennifer Neville co-correspond this work.} \\
    Purdue University\\
    \texttt{panli@purdue.edu} \\
\And
    Cong Hao \\
    Georgia Tech\\
    \texttt{callie.hao@gatech.edu} \\
\And
    Jennifer Neville \\
    Purdue University and Microsoft Research\\
    \texttt{jenneville@microsoft.com} \\
}

\begin{document}

\maketitle
\vspace{-5mm}
\begin{abstract}
  \vspace{-2mm}
Self-supervised learning of graph neural networks (GNN) is in great need because of the widespread label scarcity issue in real-world graph/network data.
Graph contrastive learning (GCL), by training GNNs to maximize the correspondence between the representations of the same graph in its different augmented forms, may yield robust and transferable GNNs even without using labels. However, GNNs trained by traditional GCL often risk capturing redundant graph features and thus may be brittle and provide sub-par performance in downstream tasks. 
Here, we propose a novel principle, termed adversarial-GCL (\textit{AD-GCL}), which enables GNNs to avoid capturing redundant information during the training by optimizing adversarial graph augmentation strategies used in GCL. 
We pair AD-GCL with theoretical explanations and design a practical instantiation based on trainable edge-dropping graph augmentation.
We experimentally validate AD-GCL\footnote{\url{https://github.com/susheels/adgcl}} by comparing with the state-of-the-art GCL methods and achieve performance gains of up-to~14\% in unsupervised, ~6\% in transfer, and~3\% in semi-supervised learning settings overall with 18 different benchmark datasets for the tasks of molecule property regression and classification, and social network classification.

\end{abstract}

\vspace{-3mm}
\section{Introduction}
\vspace{-2mm}
    \label{sec:intro}
Graph representation learning (GRL) aims to encode graph-structured data into low-dimensional vector representations, which has recently shown great potential in many applications in biochemistry, physics and social science~\cite{senior2020improved,shlomi2020graph,hamilton2020book}. Graph neural networks (GNNs), inheriting the power of neural networks~\cite{hornik1989multilayer,cybenko1989approximation}, have become the almost {\em de facto} encoders for GRL~\cite{scarselli2008graph, chami2020machine, zhang2020deep_survey, hamilton2017representation}. GNNs have been mostly studied in  cases with supervised end-to-end training~\cite{kipf2016semi,dai2016discriminative,velickovic2018graph,zhang2018end,xu2018powerful,morris2019weisfeiler,li2020distance}, where a large number of task-specific labels are needed. However, in many applications, annotating labels of graph data takes a lot of time and resources~\cite{hu2019strategies,sun2019infograph}, e.g., identifying pharmacological effect of drug molecule graphs requires living animal experiments~\cite{vogel2002drug}. Therefore, recent research efforts are directed towards studying self-supervised learning for GNNs, where only limited or even no labels are needed~\cite{kipf2016variational,grover2019graphite,gmi,dgi,sun2019infograph,you2020graph, hassani2020contrastive, xie2021self, liu2021graph,zhang2020motif, thakoor2021bootstrapped,zhu2020graph,qiu2020gcc}. 

Designing proper self-supervised-learning principles for GNNs is crucial, as they drive what information of graph-structured data will be captured by  GNNs and may heavily impact their performance in downstream tasks. Many previous works adopt the edge-reconstruction principle to match traditional network-embedding requirement~\cite{belkin2003laplacian,perozzi2014deepwalk,grover2016node2vec,ribeiro2017struc2vec}, where the edges of the input graph are expected to be reconstructed based on the output of GNNs~\cite{hamilton2017inductive,kipf2016variational,grover2019graphite}. 
Experiments showed that these GNN models learn to over-emphasize node proximity~\cite{dgi} and may lose subtle but crucial structural information, thus failing in many tasks including node-role classification~\cite{henderson2012rolx,donnat2018learning,ribeiro2017struc2vec,li2020distance} and graph classification~\cite{hu2019strategies}. 
 
To avoid the above issue, graph contrastive learning (GCL) has attracted more attention recently~\cite{xie2021self,dgi,gmi, liu2021graph,sun2019infograph,hassani2020contrastive,zhang2020motif,thakoor2021bootstrapped,zhu2020graph,qiu2020gcc}. GCL leverages the mutual information maximization principle (InfoMax)~\cite{linsker1988self} that aims to maximize the correspondence between the representations of a graph (or a node) in its different augmented forms~\cite{sun2019infograph,hassani2020contrastive,you2020graph,zhang2020motif,thakoor2021bootstrapped,zhu2020graph,qiu2020gcc}. Perfect correspondence indicates that a representation precisely identifies its corresponding graph (or node) and thus the encoding procedure does not decrease the mutual information between them.
 
However, researchers have found that the InfoMax principle may be risky because it may push encoders to capture redundant information that is irrelevant to the downstream tasks: Redundant information suffices to identify each graph to achieve InfoMax, but encoding it yields brittle representations and may severely deteriorate the performance of the encoder in the downstream tasks~\cite{tschannen2019mutual}. This observation reminds us of another principle, termed information bottleneck (IB)~\cite{tishby2000information,tishby2015deep,goldfeld2020information,alemi2016deep,vdb,betavae}. As opposed to InfoMax, IB asks the encoder to capture the \emph{minimal sufficient} information for the downstream tasks. Specifically, IB minimizes the information from the original data while maximizing the information that is relevant to the downstream tasks. As the redundant information gets removed, the encoder learnt by IB tends to be more robust and transferable. Recently, IB has been applied to GNNs~\cite{wu2020graph,yu2021recognizing}. But IB needs the knowledge of the downstream tasks that may not be available. 
 
Hence, a natural question emerges: \emph{When the knowledge of downstream tasks are unavailable, how to train GNNs that may remove redundant information?} Previous works highlight some solutions by designing data augmentation strategies for GCL but those strategies are typically task-related and sub-optimal. They either leverage domain knowledge~\cite{hassani2020contrastive,zhang2020motif,zhu2020graph}, \textit{e.g.}, node centralities in network science or molecule motifs in bio-chemistry, or depend on extensive evaluation on the downstream tasks, where the best strategy is selected based on validation performance~\cite{you2020graph, zhu2020graph}. 

\begin{figure}
    \centering
    \includegraphics[width=1.0\textwidth]{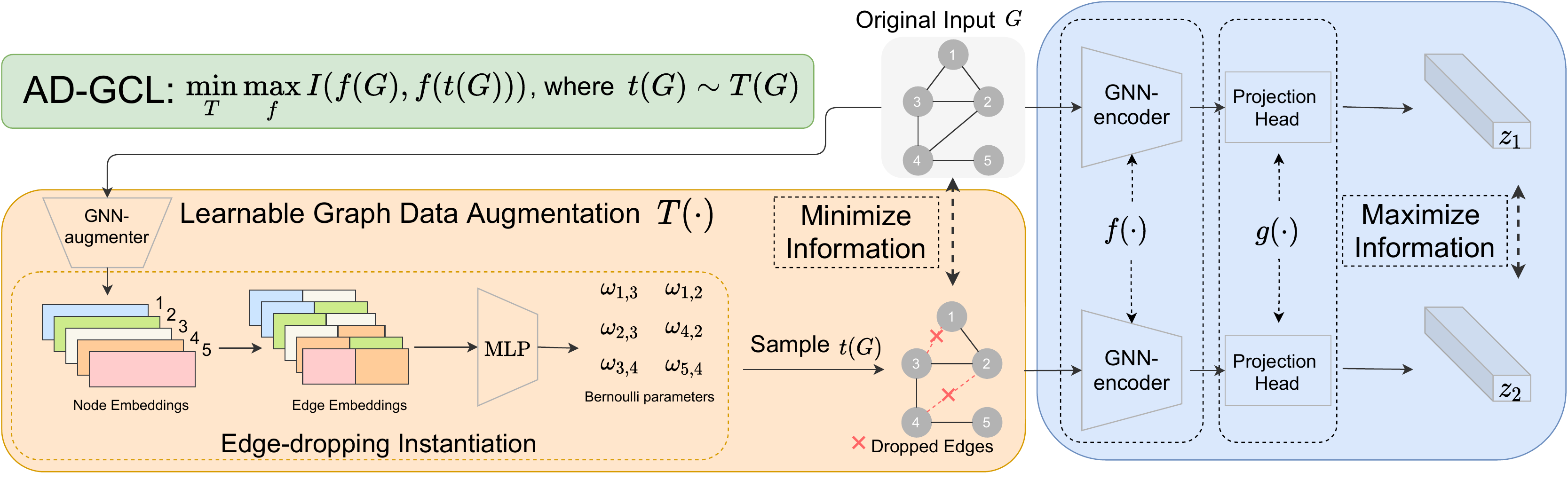}
    \vspace{-4mm}
    \caption{\small{The AD-GCL principle and its instantiation based on learnable edge-dropping augmentation. AD-GCL contains two components for graph data encoding and graph data augmentation. The GNN encoder $f(\cdot)$ maximizes the mutual information between the original graph $G$ and the augmented graph $t(G)$ while the GNN augmenter optimizes the augmentation  $T(\cdot)$ to remove the information from the original graph. The instantiation of AD-GCL proposed in this work uses edge dropping: An edge $e$ of $G$ is randomly dropped according to $\text{Bernoulli}(\omega_e)$, where $\omega_e$ is parameterized by the GNN augmenter. }}
    \label{fig:ad-gcl-pipeline}
    \vspace{-3mm}
\end{figure}

In this paper, we approach this question by proposing a novel principle that pairs GCL with adversarial training, termed \emph{AD-GCL}, as shown in Fig.\ref{fig:ad-gcl-pipeline}. We particularly focus on training self-supervised GNNs for graph-level tasks, though the idea may be generalized for node-level tasks. AD-GCL consists of two components: The first component contains a GNN encoder, which adopts InfoMax to maximize the correspondence/mutual information between the representations of the original graph and its augmented graphs. The second component contains a GNN-based augmenter, which aims to optimize the augmentation strategy to decrease redundant information from the original graph as much as possible. AD-GCL essentially allows the encoder capturing the minimal sufficient information to distinguish graphs in the dataset. We further provide theoretical explanations of AD-GCL. We show that with certain regularization on the search space of the augmenter, AD-GCL can yield a lower bound guarantee of the information related to the downstream tasks, while simultaneously holding an upper bound guarantee of the redundant information from the original graphs, which matches the aim of the IB principle. We further give an instantiation of AD-GCL: The GNN augmenter adopts a task-agnostic augmentation strategy and will learn an input-graph-dependent non-uniform-edge-drop probability to perform graph augmentation.

Finally, we extensively evaluate AD-GCL on 18 different benchmark datasets for molecule property classification and regression, and social network classification tasks in different setting viz. unsupervised learning (Sec.~\ref{exp:unsup}), transfer learning (Sec.~\ref{exp:transfer}) and semi-supervised learning (Sec.~\ref{exp:semisup}) learning. AD-GCL achieves significant performance gains in relative improvement and high mean ranks over the datasets compared to state-of-the-art baselines. 
We also study the theoretical aspects of AD-GCL with apt experiments and analyze the results to offer fresh perspectives (Sec.~\ref{exp:reg_analysis}): Interestingly, we observe that AD-GCL outperforms traditional GCL based on non-optimizable augmentation across almost the entire range of perturbation levels.  

\vspace{-3mm}
\section{Notations and Preliminaries}
\vspace{-2mm}
    We first introduce some preliminary concepts and notations for further exposition. In this work, we consider attributed graphs $G =(V,E)$ where $V$ is a node set and $E$ is an edge set. $G$ may have node attributes $\{X_{v} \in \mathbb{R}^F \mid v \in V\}$ and edge attributes $\{X_{e} \in \mathbb{R}^F \mid e \in E\}$ of dimension $F$. We denote the set of the neighbors of a node $v$ as $\mathcal{N}_v$. \vspace{-2mm}
\paragraph{Learning Graph Representations.} 
Given a set of graphs $G_i$, $i=1,2,...,n$, in some universe $\mathcal{G}$, the aim is to learn an encoder $f:\mathcal{G}\rightarrow \mathbb{R}^d$, where $f(G_i)$ can be further used in some downstream task. 
We also assume that $G_i$'s are all IID sampled from an unknown distribution $\mathbb{P}_{\mathcal{G}}$ defined over $\mathcal{G}$. 
In a downstream task, each $G_i$ is associated with a label $y_i\in\mathcal{Y}$. Another model $q:\mathbb{R}^d\rightarrow \mathcal{Y}$ will be learnt to predict $Y_i$ based on $q(f(G_i))$. We assume $(G_i, Y_i)$'s are IID  sampled from a distribution $\mathbb{P}_{\mathcal{G}\times \mathcal{Y}} = \mathbb{P}_{\mathcal{Y}|\mathcal{G}}\mathbb{P}_{\mathcal{G}}$, where $\mathbb{P}_{\mathcal{Y}|\mathcal{G}}$ is the conditional distribution of the graph label in the downstream task given the graph. 
\paragraph{Graph Neural Networks (GNNs).} In this work, we focus on using GNNs, message passing GNNs in particular~\cite{gilmer2017neural}, as the encoder $f$. For a graph $G=(V,E)$, every node $v\in V$ will be paired with a node representation $h_v$ initialized as $h_v^{(0)}=X_v$. These representations will be updated by a GNN. During the $k^{\text{th}}$ iteration, each $h_v^{(k-1)}$ is updated using $v'$s neighbourhood information expressed as,

\vspace{-4mm}
\begin{equation}
\label{eq:gnn_iteration}
   \small h_v^{(k)} = \text{UPDATE}^{(k)} \Bigg( h_v^{(k-1)}, \; \text{AGGREGATE}^{(k)} \Big( \big\{(h_u^{(k-1)}, X_{uv})  \mid u \in \mathcal{N}_v\big\} \Big) \Bigg)
\end{equation}
 where $\text{AGGREGATE}(\cdot)$ is a trainable function that maps the set of node representations and edge attributes $X_{uv}$ to an aggregated vector, $\text{UPDATE}(\cdot)$ is another trainable function that maps both $v\text{'s}$ current representation and the aggregated vector to $v\text{'s}$ updated representation.
After $K$ iterations of Eq.~\ref{eq:gnn_iteration}, the graph representation is obtained by pooling the final set of node representations as,
\begin{equation}
    \label{eq:pooling}
    f(G):\triangleq h_{G} = \text{POOL}\big(\{ h_v^{(K)} \mid v \in V \}\big)
\end{equation}
For design choices regarding aggregation, update and pooling functions we refer the reader to \cite{zhang2020deep_survey, chami2020machine, hamilton2020book}. \vspace{-5mm}
\paragraph{The Mutual Information Maximization Principle.} 
GCL is built upon the InfoMax principle~\cite{linsker1988self}, which prescribes to learn an encoder $f$ that maximizes the mutual information or the correspondence between the graph and its representation. The rationale behind GCL is that a graph representation $f(G)$ should capture the features of the graph $G$ so that representation can distinguish this graph from other graphs. Specifically, the objective of GCL follows
\begin{equation} \label{eq:infomax}
   \text{InfoMax:}\quad \max_{f} I(G ; f(G)), \quad \text{where } G\sim \mathbb{P}_{\mathcal{G}}.
\end{equation}
where $I(X_1;X_2)$ denotes the mutual information between two random variables $X_1$ and $X_2$~\cite{cover2012elements}. 

Note that the encoder $f(\cdot)$ given by GNNs is not injective in the graph space $\mathcal{G}$ due to its limited expressive power~\cite{xu2018powerful,morris2019weisfeiler}. Specifically, for the graphs that cannot be distinguished by $1$-WL test~\cite{weisfeiler1968reduction}, GNNs will associate them with the same representations. We leave more discussion on 1-WL test in Appendix~\ref{apd:1wl}. In contrast to using CNNs as encoders, one can never expect GNNs to identify all the graphs in $\mathcal{G}$ based their representations, which introduces a unique challenge for GCL. 

\vspace{-3mm}
\section{Adversarial Graph Contrastive Learning}
\vspace{-2mm}
    In this section, we introduce our adversarial graph contrastive learning (AD-GCL) framework and one of its instantiations based on edge perturbation. 
\vspace{-3mm}
\subsection{Theoretical Motivation and Formulation of AD-GCL}
\vspace{-1mm}
\label{sec:ad-gcl}

The InfoMax principle in Eq.~\ref{eq:infomax} could be problematic in practice for general representation learning. Tschannen et al. have shown that for image classification, representations capturing the information that is entirely irrelevant to the image labels are also able to maximize the mutual information but such representations are definitely not useful for image classification~\cite{tschannen2019mutual}. A similar issue can also be observed 
in graph representation learning, as illustrated by Fig.\ref{fig:bad-exm}: We consider a binary graph classification problem with graphs in the dataset ogbg-molbace~\cite{hu2020open}. Two GNN encoders with exactly the same architecture are trained to keep mutual information maximization between graph representations and the input graphs, but one of the GNN encoders in the same time is further supervised by random graph labels. Although the GNN encoder supervised by random labels still keeps one-to-one correspondance between every input graph and its representation (i.e., mutual information maximization), we may observe significant performance degeneration of this GNN encoder when evaluating it over the downstream ground-truth labels.  
More detailed experiment setup is left in Appendix~\ref{apd:exp_settings_motivation}. 

\begin{wrapfigure}{r}{0.38\textwidth}
\vspace{-7mm}
  \begin{center}
    \includegraphics[width=0.38\textwidth]{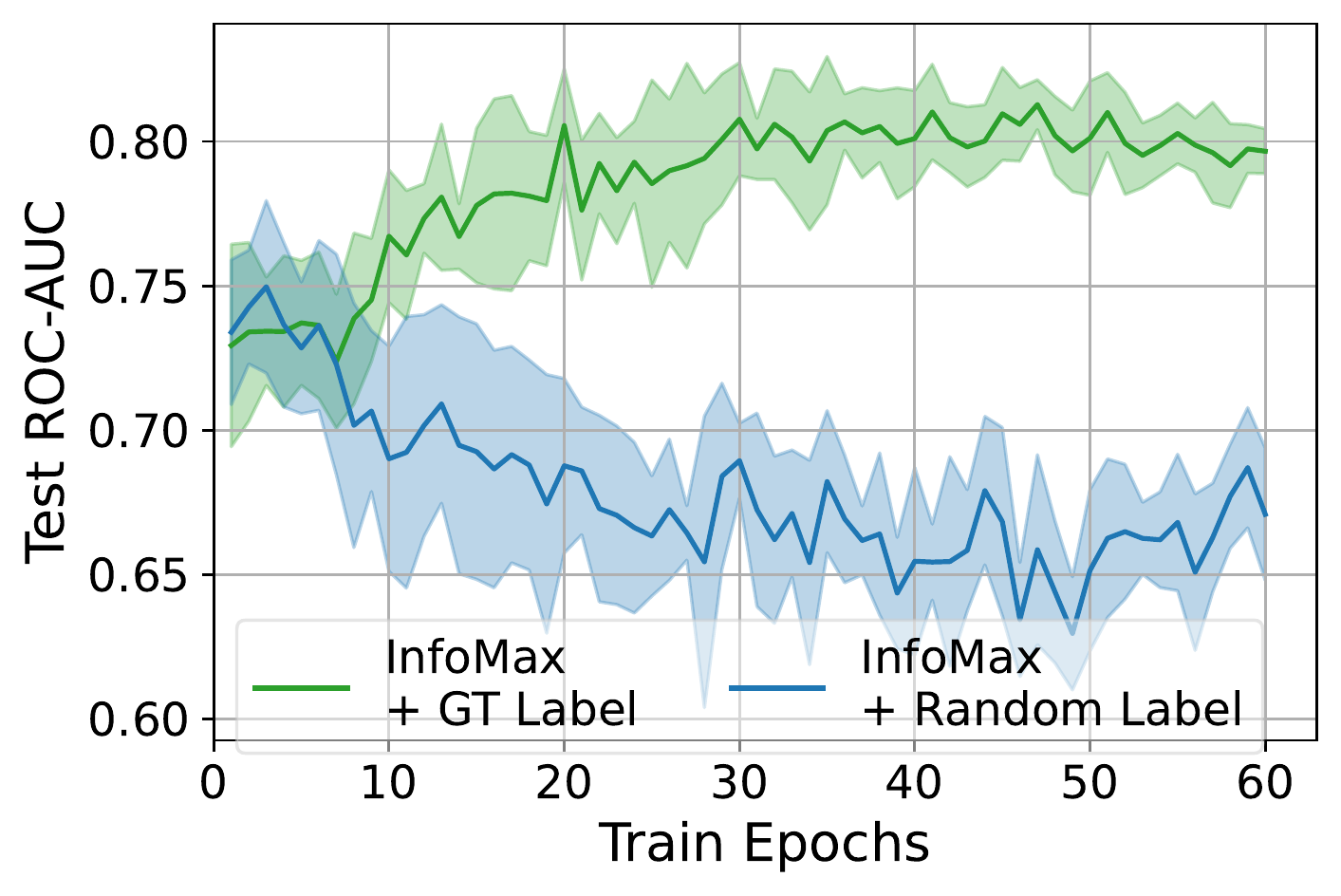}
  \end{center}
  \vspace{-4mm}
  \caption{\small{Two GNNs keep the mutual information maximized between graphs and their representations. Simultaneously, they get supervised by ground-truth labels (green) and random labels (blue) respectively. 
  The curves show their testing performance on predicting ground-truth labels. }} 
  \vspace{-2mm}
  \label{fig:bad-exm}
\end{wrapfigure}

This observation inspires us to rethink what a good graph representation is. Recently, the information bottleneck has applied to learn graph representations~\cite{wu2020graph,yu2021recognizing}. Specifically, the objective of graph information bottleneck (GIB) follows \vspace{-1mm}
\begin{align} \label{eq:GIB}
  \text{GIB:}\quad  \max_{f} I(f(G); Y) - \beta I(G; f(G)),
\end{align} 
\text{where } $(G,Y)\sim \mathbb{P}_{\mathcal{G}\times \mathcal{Y}}, \beta \text{ is a positive constant.}$ Comparing Eq.~\ref{eq:infomax} and Eq.~\ref{eq:GIB}, we may observe the different requirements between InfoMax and GIB: InfoMax asks for maximizing the information from the original graph, while GIB asks for minimizing such information but simultaneously maximizing the information that is relevant to the downstream tasks. As GIB asks to remove redundant information, GIB naturally avoids the issue encountered in Fig.\ref{fig:bad-exm}. Removing extra information also makes GNNs trained w.r.t. GIB robust to adverserial attack and strongly transferrable~\cite{wu2020graph,yu2021recognizing}.  

Unfortunately, GIB requires the knowledge of the class labels $Y$ from the downstream task and thus does not apply to self-supervised training of GNNs where there are few or no labels. Then, the question is how to learn robust and transferable GNNs in a self-supervised way. 

To address this, we will develop a GCL approach that uses adversarial learning to  avoid capturing redundant information during the representation learning. 
In general, GCL methods use graph data augmentation (GDA) processes to perturb the original observed graphs and decrease the amount of information they encode. 
Then, the methods apply InfoMax over perturbed graph pairs (using different GDAs) to train an encoder $f$ to capture the remaining information. 
\begin{definition}[Graph Data Augmentation (GDA)]
For a graph $G \in \mathcal{G}$, $T(G)$ denotes a graph data augmentation of $G$, which is a distribution defined over $\mathcal{G}$ conditioned on $G$. We use $t(G)\in \mathcal{G}$ to denote a sample of $T(G)$. 
\end{definition}\vspace{-2mm}
Specifically, given two ways of GDA $T_1$ and $T_2$, the objective of GCL becomes
\begin{equation} \label{eq:gcl-fix-aug}
  \text{GDA-GCL:}\,  \max_{f} I(f(t_1(G)) ; f(t_2(G))), \, \text{where } G\sim \mathbb{P}_{\mathcal{G}}, t_i(G)\sim T_i(G), i\in\{1,2\}.
\end{equation}
In practice, GDA processes are often pre-designed based on either domain knowledge or extensive evaluation, and improper choice of GDA may severely impact the downstream performance~\cite{you2020graph,hu2019strategies}. 
 We will review a few GDAs adopted in existing works in Sec.\ref{sec:related}. 
 
 In contrast to previous predefined GDAs, our idea, inspired by GIB, is to {\em learn} the GDA process (over a parameterized family), so that the encoder $f$ can capture the \textbf{minimal information } that is sufficient to identify each graph. 
 
\vspace{-1mm}

\paragraph{AD-GCL:} 
We optimize the following objective, over a GDA family $\mathcal{T}$ (defined below).
\begin{equation} \label{eq:ad-gcl}
  \text{AD-GCL:}\quad  \min_{T\in \mathcal{T}} \max_{f} I(f(G) ; f(t(G))), \quad \text{where } G\sim \mathbb{P}_{\mathcal{G}}, t(G)\sim T(G),
\end{equation}
\begin{definition}[Graph Data Augmentation Family]  Let $\mathcal{T}$ denote a family of different GDAs $T_{\Phi}(\cdot)$, where $\Phi$ is the parameter in some universe. A $T_{\Phi}(\cdot)\in \mathcal{T}$ is a specific GDA with parameter $\Phi$. 
\end{definition} \vspace{-0.5mm}
The min-max principle in AD-GCL aims to train the encoder such that even with a very aggressive GDA (i.e., where $t(G)$ is very different from $G$), the mutual information / the correspondence between the perturbed graph and the original graph can be maximized. 
Compared with the two GDAs adopted in GDA-GCL (Eq.\ref{eq:gcl-fix-aug}), AD-GCL views the original graph $G$ as the anchor while pushing its perturbation $T(G)$ as far from the anchor as it can. The automatic search over $T \in \mathcal{T}$ saves a great deal of effort  evaluating different combinations of GDA as adopted in \cite{you2020graph}. 

\vspace{-2mm}
\paragraph{Relating AD-GCL to the downstream task.} 
Next, we will theoretically characterize the property of the encoder trained via AD-GCL.  
The analysis here not only further illustrates the rationale of AD-GCL but helps design practical $\mathcal{T}$ when some knowledge of $Y$ is accessible. But note that our analysis does not make any assumption on the availability of $Y$.

Note that GNNs learning graph representations is very different from CNNs learning image representations because GNNs are never injective mappings between the graph universe $\mathcal{G}$ and the representation space $\mathbb{R}^d$, because the expressive power of GNNs is limited by the 1-WL test~\cite{weisfeiler1968reduction,xu2018powerful,morris2019weisfeiler}. So, we need to define a quotient space of $\mathcal{G}$ based on the equivalence given by the 1-WL test. \vspace{-1mm}

\begin{definition}[Graph Quotient Space]
Define the equivalence $\cong$ between two graphs $G_1\cong G_2$ if $G_1,\,G_2$ cannot be distinguished by the 1-WL test. Define the quotient space $\mathcal{G}' = \mathcal{G}/\cong$. 
\end{definition}\vspace{-1mm}
So every element in the quotient space, i.e., $G' \in \mathcal{G}'$, is a representative graph from a family of graphs that cannot be distinguished by the 1-WL test. Note that our definition also allows attributed graphs.
\begin{definition}[Probability Measures in $\mathcal{G}'$]
Define $\mathbb{P}_{\mathcal{G}'}$ over the space $\mathcal{G}'$ such that $\mathbb{P}_{\mathcal{G}'}(G') = \mathbb{P}_{\mathcal{G}}(G\cong G')$ for any $G'\in\mathcal{G}'$. Further define $\mathbb{P}_{\mathcal{G}'\times \mathcal{Y}} (G', Y') = \mathbb{P}_{\mathcal{G}\times \mathcal{Y}} (G\cong G', Y=Y')$. Given a GDA $T(\cdot)$ defined over $\mathcal{G}$, define a distribution on $\mathcal{G}'$, $T'(G') = \mathbb{E}_{G\sim \mathbb{P}_{\mathcal{G}}}[T(G)| G\cong G']$ for $G'\in \mathcal{G}'$. 
\end{definition}\vspace{-1mm}


Now, we provide our theoretical results and give their implication. The proof is in the Appendix~\ref{apd:prf}.
\begin{theorem} \label{thm:main}
Suppose the encoder $f$ is implemented by a GNN as powerful as the 1-WL test. Suppose $\mathcal{G}$ is a countable space and thus $\mathcal{G'}$ is a countable space. Then, the optimal solution $(f^*, T^*)$ to AD-GCL satisfies, letting $T'^{*}(G') = \mathbb{E}_{G\sim \mathbb{P}_{\mathcal{G}}}[T^*(G)| G\cong G']$, \vspace{-1mm}
\begin{enumerate}[leftmargin=*]
    \item $I(f^*(t^*(G)); G\,|\,Y) \leq  \min_{T\in\mathcal{T}} I(t'(G'); G') - I(t'^*(G'); Y) $, where $t'(G')\sim T'(G')$, $t'^*(G')\sim T'^{*}(G')$, $(G,Y)\sim \mathbb{P}_{\mathcal{G}\times \mathcal{Y}}$ and $(G',Y)\sim \mathbb{P}_{\mathcal{G}'\times \mathcal{Y}}$.
    \item $I(f^*(G); Y)\geq I(f^*(t'^*(G')); Y) = I(t'^*(G'); Y)$, where  $t'^*(G')\sim T'^*(G')$, $(G,Y)\sim \mathbb{P}_{\mathcal{G}\times \mathcal{Y}}$ and $(G',Y)\sim \mathbb{P}_{\mathcal{G}'\times \mathcal{Y}}$.
\end{enumerate}
\vspace{-1mm}
\end{theorem}

The statement 1 in Theorem~\ref{thm:main} guarantees a upper bound of the information that is captured by the representations but irrelevant to the downstream task, which matches our aim. This bound has a form very relevant to the GIB principle (Eq.\ref{eq:GIB} when $\beta=1$), since $ \min_{T\in\mathcal{T}} I(t'(G'); G') - I(t'^*(G'); Y) \geq \min_{f} [I(f(G); G) - I(f(G); Y)]$,
where $f$ is a GNN encoder as powerful as the 1-WL test. But note that this inequality also implies that the encoder given by AD-GCL may be worse than the optimal encoder given by GIB ($\beta =1$). This makes sense as GIB has the access to the downstream task $Y$.

The statement 2 in Theorem~\ref{thm:main} guarantees a lower bound of the mutual information between the learnt representations and the labels of the downstream task. As long as the GDA family $\mathcal{T}$ has a good control, $I(t'^*(G'); Y)\geq \min_{T\in\mathcal{T}}I(t'(G');Y)$ and $I(f^*(G); Y)$ thus cannot be too small. This implies that it is better to regularize when learning over $\mathcal{T}$. In our instantiation, based on edge-dropping augmentation (Sec.~\ref{sec:instantiation}), we regularize the ratio of dropped edges per graph. 

\vspace{-1mm}
\subsection{Instantiation of AD-GCL via Learnable Edge Perturbation} \label{sec:instantiation}
\vspace{-1mm}
We now introduce a practical instantiation of the AD-GCL principle (Eq.~\ref{eq:ad-gcl}) based on learnable edge-dropping augmentations as illustrated in Fig.~\ref{fig:ad-gcl-pipeline}. 
(See Appendix~\ref{apd:algo} for a summary of AD-GCL in its algorithmic form.)
The objective of AD-GCL has two folds: (1) Optimize the encoder $f$ to maximize the mutual information between the representations of the original graph $G$ and its augmented graph $t(G)$; (2) Optimize the GDA $T(G)$ where $t(G)$ is sampled to minimize such a mutual information. We always set the encoder as a GNN $f_{\Theta}$ with learnable parameters $\Theta$ and next we focus on the GDA, $T_{\Phi}(G)$ that has learnable parameters $\Phi$. \vspace{-1mm}

\paragraph{Learnable Edge Dropping GDA model $T_{\Phi}(\cdot)$.}
Edge dropping is the operation of deleting some edges in a graph. As a proof of concept, we adopt edge dropping to formulate the GDA family $\mathcal{T}$. Other types of GDAs such as node dropping, edge adding and feature masking can also be paired with our AD-GCL principle. 
Interestingly, in our experiments, edge-dropping augmentation optimized by AD-GCL has already achieved much better performance than any pre-defined random GDAs even carefully selected via extensive evaluation~\cite{you2020graph} (See Sec.\ref{sec:exp}). 
Another reason that supports edge dropping is due to our Theorem~\ref{thm:main} statement 2, which shows that good GDAs should keep some information related to the downstream tasks. Many GRL downstream tasks such as molecule classification only depends on the structural fingerprints that can be represented as subgraphs of the original graph~\cite{duvenaud2015convolutional}. Dropping a few edges may not change those subgraph structures 
and thus keeps the information sufficient to the downstream classification. But note that this reasoning does not mean that we leverage domain knowledge to design GDA, as the family $\mathcal{T}$ is still broad and the specific GDA still needs to be optimized. Moreover, experiments show that our instantiation also works extremely well on social network classification and molecule property regression, where the evidence of subgraph fingerprints may not exist any more.  \vspace{-1mm}

\paragraph{Parameterizing $T_{\Phi}(\cdot)$.} For each $G=(V,E)$, we set $T_{\Phi}(G)$, $T\in \mathcal{T}$ as a random graph model \cite{gilbert1959random, erdds1959random} conditioning on $G$. Each sample $t(G)\sim T_{\Phi}(G)$ is a graph that shares the same node set with $G$ while the edge set of $t(G)$ is only a subset of $E$. Each edge $e \in E$ will be associated with a random variable $p_e\sim$ Bernoulli$(\omega_e)$, where $e$ is in $t(G)$ if $p_e=1$ and is dropped otherwise. 

We parameterize the Bernoulli weights $\omega_e$ by leveraging another GNN, \textit{i.e.,} the \emph{augmenter}, to run on $G$ according to Eq.\ref{eq:gnn_iteration} of $K$ layers, get the final-layer node representations $\{h_v^{(K)}|v\in V\}$ and set
\begin{equation}\label{eq:w_e}
    \omega_{e} = \text{MLP} ([h_u^{(K)};h_z^{(K)}]), \quad \text{where }  e=(u,z) \,\text{and}\,\{ h_v^{(K)} \mid v \in V \} = \text{GNN-augmenter}(G)
\end{equation}

To train $T(G)$ in an end-to-end fashion, 
we relax the discrete $p_{e}$ to be a continuous variable in $[0,1]$ and utilize the Gumbel-Max reparametrization trick \cite{maddison2016concrete, jang2016categorical}. Specifically, 
$p_{e} = \text{Sigmoid}((\log \delta - \log(1-\delta) + \omega_{e})/\tau)$, where $\delta \sim \text{Uniform(0,1)}$. 
As temperature hyper-parameter $\tau \rightarrow 0$, $p_{e}$ gets closer to being binary. Moreover, the gradients $\frac{\partial p_{e}}{\partial \omega_{e}}$ are smooth and well defined. This style of edge dropping based on a random graph model has also been used for parameterized explanations of GNNs \cite{luo2020parameterized}. \vspace{-1mm}

\paragraph{Regularizing $T_{\Phi}(\cdot)$.} As shown in Theorem~\ref{thm:main}, a reasonable GDA should keep a certain amount of information related to the downstream tasks (statement 2). Hence, we expect the GDAs in the edge dropping family $\mathcal{T}$ not to perform very aggressive perturbation. Therefore, we regularize the ratio of edges being dropped  per graph by enforcing the following constraint: For a graph $G$ and its augmented graph $t(G)$, we add $\sum_{e\in E} \omega_e/|E|$ to the objective, where $\omega_e$ is defined in Eq.\ref{eq:w_e} indicates the probability that $e$ gets dropped. 

Putting everything together, the final objective is as follows.
\begin{equation} \label{eq:ad-gcl-encoder-augmentor-reg}
  \min_{\Phi} \max_{\Theta} I(f_{\Theta}(G) ; f_{\Theta}(t(G))) + \lambda_{\text{reg}} \mathbb{E}_{G}\big[\sum_{e \in E} \omega_e/|E|\big],\, \text{where } G\sim \mathbb{P}_{\mathcal{G}}, t(G)\sim T_{\Phi}(G).
\end{equation}

\vspace{-3mm}
Note $\Phi$ corresponds to the learnable parameters of the augmenter GNN and MLP used to derive the $\omega_{e}$'s and $\Theta$ corresponds to the learnable parameters of the GNN $f$. 
\vspace{-2mm}
\paragraph{Estimating the objective in Eq.\ref{eq:ad-gcl-encoder-augmentor-reg}.} In our implementation, the second (regularization) term is easy to estimate empirically. For the first (mutual information) term, we adopt InfoNCE as the estimator~\cite{oord2018representation, tian2019contrastive, poole2019variational}, which is known to be a lower bound of the mutual information and is frequently used for contrastive learning~\cite{oord2018representation,tschannen2019mutual,chen2020simple}. Specfically, during the training, given a minibatch of $m$ graphs $\{G_i\}_{i=1}^m$, let $z_{i, 1} = g(f_{\Theta}(G_i))$ and $z_{i, 2} = g(f_{\Theta}(t(G_i)))$ where $g(\cdot)$ is the projection head implemented by a 2-layer MLP as suggested in~\cite{chen2020simple}. With $sim(\cdot, \cdot)$ denoting cosine similarity, we estimate the mutual information for the mini-batch as follows.
\begin{equation}
    \label{eq:loss}
    \small{I(f_{\Theta}(G) ; f_{\Theta}(t(G))) \rightarrow \hat{I}= \frac{1}{m}\sum_{i= 1}^{m} \log \frac{\exp(sim(z_{i,1}, z_{i,2}))}{\sum_{i^\prime = 1, i^\prime \neq i }^{m}\exp(sim(z_{i,1}, z_{i^\prime,2}))}}
\end{equation}

\vspace{-6mm}
\section{Related Work}
\vspace{-2mm}
    \label{sec:related}
GNNs for GRL is a broad field and gets a high-level review in the Sec.~\ref{sec:intro}. Here, we focus on the topics that are most relevant to graph contrastive learning (GCL).

Contrastive learning (CL)~\cite{linsker1988self,becker1992self, henaff2020data, oord2018representation, tian2019contrastive,hjelm2018learning} was initially proposed to train CNNs for image representation learning and has recently achieved great success~\cite{chen2020simple,chen2020big}. GCL applies the idea of CL on GNNs. In contrast to the case of CNNs, GCL trained using GNNs posts us new fundamental challenges. An image often has multiple natural views, say by imposing different color filters and so on. Hence, different views of an image give natural contrastive pairs for CL to train CNNs. However, graphs are more abstract and the irregularity of graph structures typically provides crucial information. Thus, designing contrastive pairs for GCL must play with irregular graph structures and thus becomes more challenging. Some works use different parts of a graph to build contrastive pairs, including nodes \textit{v.s.} whole graphs~\cite{velivckovic2018deep,sun2019infograph}, nodes \textit{v.s.} nodes~\cite{peng2020graph}, nodes \textit{v.s.} subgraphs~\cite{jiao2020sub, hu2019strategies}.
Other works adopt graph data augmentations (GDA) such as edge perturbation~\cite{qiu2020gcc} to generate contrastive pairs. Recently. GraphCL~\cite{you2020graph} gives an extensive study on different combinations of GDAs including node dropping, edge perturbation, subgraph sampling and feature masking. Extensive evaluation is required to determine good combinations. MVGRL~\cite{hassani2020contrastive} and GCA~\cite{zhu2020graph} leverage the domain knowledge of network science and adopt network centrality to perform GDAs. Note that none of the above methods consider optimizing augmentations.  In contrast, our principle AD-GCL provides theoretical guiding principles to optimize augmentations. Very recently, JOAO~\cite{you2021graph} adopts a bi-level optimization framework sharing some high-level ideas with our adversarial training strategy but has several differences: 1) the GDA search space in JOAO is set as different types of augmentation with uniform perturbation, such as uniform edge/node dropping while we allow augmentation with non-uniform perturbation. 2) JOAO relaxes the GDA combinatorial search problem into continuous space via Jensen’s inequality and adopts projected gradient descent to optimize. Ours, instead, adopts Bayesian modeling plus reparameterization tricks to optimize. The performance comparison between AD-GCL and JOAO for the tasks investigated in Sec.~\ref{sec:exp} is given in Appendix~\ref{apd:joao_compare}. 


Tian et al.~\cite{tian2020makes} has recently proposed the InfoMin principle that shares some ideas with AD-GCL but there are several fundamental differences. Theoretically, InfoMin needs the downstream tasks to supervise the augmentation. Rephrased in our notation, the optimal augmentation $T_{IM}(G)$ given by InfoMin (called the sweet spot in~\cite{tian2020makes}) needs to satisfy $I(t_{IM}(G);Y) = I(G;Y)$ and $I(t_{IM}(G);G|Y)=0$, $t_{IM}(G)\sim T_{IM}(G)$, neither of which are possible without the downstream-task knowledge. Instead, our Theorem~\ref{thm:main} provides more reasonable arguments and creatively suggests using regularization to control the tradeoff. Empirically, InfoMin is applied to CNNs while AD-GCL is applied to GNNs. AD-GCL needs to handle the above challenges due to irregular graph structures and the limited expressive power of GNNs~\cite{xu2018powerful,morris2019weisfeiler}, which InfoMin does not consider.

\vspace{-3mm}
\section{Experiments and Analysis}
\vspace{-3mm}
    \label{sec:exp}
This section is devoted to the empirical evaluation of the proposed instantiation of our AD-GCL principle. Our initial focus is on unsupervised learning which is followed by analysis of the effects of regularization. We further apply AD-GCL to transfer and semi-supervised learning. Summary of datasets and training details for specific experiments are provided in Appendix~\ref{apd:datasets} and \ref{apd:exp_settings} respectively.

\vspace{-3mm}
\subsection{Unsupervised Learning}
\vspace{-3mm}
\label{exp:unsup}
In this setting, an encoder (specifically GIN~\cite{xu2018how}) is trained with different self-supervised methods to learn graph representations, which are then evaluated by feeding these representations to make prediction for the downstream tasks. We use datasets from Open Graph Benchmark (OGB)~\cite{hu2020open}, TU Dataset~\cite{Morris+2020} and ZINC~\cite{dwivedi2020benchmarking} for graph-level property classification and regression. More details regarding the experimental setting are provided in the Appendix~\ref{apd:exp_settings}.

We consider two types of AD-GCL, where one is with a fixed regularization weight $\lambda_{\text{reg}}=5$ (Eq.\ref{eq:ad-gcl-encoder-augmentor-reg}), termed AD-GCL-FIX, and another is with $\lambda_{\text{reg}}$ tuned over the validation set among $\{0.1, 0.3, 0.5, 1.0, 2.0, 5.0, 10.0\}$, termed AD-GCL-OPT. AD-GCL-FIX assumes any information from the downstream task as unavailable while AD-GCL-OPT assumes the augmentation search space has some weak information from the downstream task. A full range of analysis on how $\lambda_{\text{reg}}$ impacts AD-GCL will be investigated in Sec.~\ref{exp:reg_analysis}. We compare AD-GCL with three unsupervised/self-supervised learning baselines for graph-level tasks, which include randomly initialized untrained GIN (RU-GIN)~\cite{xu2018how}, InfoGraph~\cite{sun2019infograph} and GraphCL~\cite{you2020graph}. Previous works~\cite{you2020graph, sun2019infograph} show that they generally outperform graph kernels~\cite{kriege2020survey,yanardag2015deep,shervashidze2011weisfeiler} and network embedding methods~\cite{grover2016node2vec, perozzi2014deepwalk, narayanan2017graph2vec,adhikari2018sub2vec}.

We also adopt GCL with GDA based on non-adversarial edge dropping (NAD-GCL) for ablation study. NAD-GCL drops the edges of a graph uniformly at random. We consider NAD-GCL-FIX and NAD-GCL-OPT with different edge drop ratios. NAD-GCL-GCL adopts the edge drop ratio of AD-GCL-FIX at the saddle point of the optimization (Eq.\ref{eq:ad-gcl-encoder-augmentor-reg}) while NAD-GCL-OPT optimally tunes the edge drop ratio over the validation datasets to match AD-GCL-OPT.
We also adopt fully supervised GIN (F-GIN) to provide an anchor of the performance. We stress that all methods adopt GIN~\cite{xu2018how} as the encoder. Except F-GIN, all methods adopt a downstream \emph{linear} classifier or regressor with the same hyper-parameters for fair comparison. Adopting \emph{linear models} was suggested by~\cite{tschannen2019mutual}, which explicitly attributes any performance gain/drop to the quality of learnt representations.  

\begin{table}[t]
\centering
\renewcommand{\arraystretch}{1.5}
\newcommand{\STAB}[1]{\begin{tabular}{@{}c@{}}#1\end{tabular}}
\resizebox{\textwidth}{!}{%
\begin{tabular}{clcccc|ccccc}
\hline
& Dataset                                      & NCI1           & PROTEINS            & MUTAG          & DD                                      & COLLAB           & RDT-B            & RDT-M5K          & IMDB-B           & IMDB-M           \\ \hline

& F-GIN & 78.27 $\pm$ 1.35 & 72.39 $\pm$ 2.76 & 90.41 $\pm$ 4.61 & 74.87 $\pm$ 3.56 & 74.82 $\pm$ 0.92 & 86.79 $\pm$ 2.04 & 53.28 $\pm$ 3.17 & 71.83 $\pm$ 1.93 & 48.46 $\pm$ 2.31 \\ \hline
\multirow{3}{*}{\STAB{\rotatebox[origin=c]{90}{Baselines}}}&RU-GIN~\cite{xu2018how}& 62.98 $\pm$ 0.10 & 69.03 $\pm$ 0.33 & 87.61 $\pm$ 0.39 & 74.22 $\pm$ 0.30  & 63.08 $\pm$ 0.10 & 58.97 $\pm$ 0.13 & 27.52 $\pm$ 0.61 & 51.86 $\pm$ 0.33 & 32.81 $\pm$ 0.57 \\
& InfoGraph~\cite{sun2019infograph}& 68.13 $\pm$ 0.59 & 72.57 $\pm$ 0.65 & 87.71 $\pm$ 1.77 & 75.23 $\pm$ 0.39                                     & 70.35 $\pm$ 0.64 & 78.79 $\pm$ 2.14 & 51.11 $\pm$ 0.55 & 71.11 $\pm$ 0.88 & 48.66 $\pm$ 0.67 \\
& GraphCL ~\cite{you2020graph}& 68.54 $\pm$ 0.55 & 72.86 $\pm$ 1.01 & 88.29 $\pm$ 1.31 & 74.70 $\pm$ 0.70                                     & 71.26 $\pm$ 0.55 & 82.63 $\pm$ 0.99 & 53.05 $\pm$ 0.40 & 70.80 $\pm$ 0.77 & 48.49 $\pm$ 0.63 \\\hline
\multirow{2}{*}{\STAB{\rotatebox[origin=c]{90}{AB-S}}}& NAD-GCL-FIX & 69.23 $\pm$ 0.60 & 72.81 $\pm$ 0.71 & 88.58 $\pm$ 1.58 & 74.55 $\pm$ 0.55              & 71.56 $\pm$ 0.58 & 83.41 $\pm$ 0.66 & 52.72 $\pm$ 0.71 & 70.94 $\pm$ 0.77 & 48.33 $\pm$ 0.47 \\
& NAD-GCL-OPT & 69.30 $\pm$ 0.32 & 73.18 $\pm$ 0.71 & 89.05 $\pm$ 1.06 & 74.55 $\pm$ 0.55             & 72.04 $\pm$ 0.67 & 83.74 $\pm$ 0.76 & 53.43 $\pm$ 0.26 & 71.94 $\pm$ 0.59 & 49.01 $\pm$ 0.93 \\\hline\hline
\multirow{2}{*}{\STAB{\rotatebox[origin=c]{90}{Ours}}} 

& AD-GCL-FIX & \bf{69.67 $\pm$ 0.51}$^\star$ & \bf{73.59 $\pm$ 0.65} & \bf{89.25 $\pm$ 1.45} & 74.49 $\pm$ 0.52 & \textbf{73.32 $\pm$ 0.61}$^\star$ & \textbf{85.52 $\pm$ 0.79}$^\star$ & 53.00 $\pm$ 0.82 & \textbf{71.57 $\pm$ 1.01} & \textbf{49.04 $\pm$ 0.53} \\ 

& AD-GCL-OPT & \bf{69.67 $\pm$ 0.51}$^\star$ & \bf{73.81 $\pm$ 0.46}$^\star$ & \textbf{89.70 $\pm$ 1.03} & 75.10 $\pm$ 0.39 & \textbf{73.32 $\pm$ 0.61}$^\star$ & \textbf{85.52 $\pm$ 0.79}$^\star$ & \textbf{54.93 $\pm$ 0.43}$^\star$ & \textbf{72.33 $\pm$ 0.56}$^\star$ & \textbf{49.89 $\pm$ 0.66}$^\star$ \\ \hline\hline
\end{tabular}%
}\\
\vspace{3mm}
\resizebox{\textwidth}{!}{%
\begin{tabular}{clcccc|ccccc}
\hline
& Task                                         & \multicolumn{4}{c|}{Regression (Downstream Classifier - Linear Regression + L2)} & \multicolumn{5}{c}{Classification (Downstream Classifier - Logistic Regression + L2)}       \\ \cline{1-11} 
& Dataset                                      & molesol            & mollipo           & molfreesolv        & \multicolumn{1}{|c|}{ZINC-10K}           & molbace          & molbbbp          & molclintox       & moltox21         & molsider         \\
& Metric                                       & \multicolumn{3}{c}{RMSE (shared) ($\downarrow$)}                     & \multicolumn{1}{|c|}{MAE ($\downarrow$) }& \multicolumn{5}{c}{ROC-AUC \% (shared) ($\uparrow$)}                                               \\ \hline

&F-GIN                         & 1.173 $\pm$ 0.057  & 0.757 $\pm$ 0.018 & 2.755 $\pm$ 0.349  & \multicolumn{1}{|c|}{0.254 $\pm$ 0.005}  & 72.97 $\pm$ 4.00 & 68.17 $\pm$ 1.48 & 88.14 $\pm$ 2.51 & 74.91 $\pm$ 0.51 & 57.60 $\pm$ 1.40 \\ \hline
\multirow{3}{*}{\STAB{\rotatebox[origin=c]{90}{Baselines}}} & RU-GIN~\cite{xu2018how}                              & 1.706 $\pm$ 0.180  & 1.075 $\pm$ 0.022 & 7.526 $\pm$ 2.119  & \multicolumn{1}{|c|}{0.809 $\pm$ 0.022}  & 75.07 $\pm$ 2.23 & 64.48 $\pm$ 2.46 & 72.29 $\pm$ 4.15 & 71.53 $\pm$ 0.74 & 62.29 $\pm$ 1.12 \\

&InfoGraph~\cite{sun2019infograph}                                    & 1.344 $\pm$ 0.178  & 1.005 $\pm$ 0.023 & 10.005 $\pm$ 4.819 & \multicolumn{1}{|c|}{0.890 $\pm$ 0.017}  & 74.74 $\pm$ 3.64 & 66.33 $\pm$ 2.79 & 64.50 $\pm$ 5.32 & 69.74 $\pm$ 0.57 & 60.54 $\pm$ 0.90 \\

&GraphCL~\cite{you2020graph}                                      & 1.272 $\pm$ 0.089  & 0.910 $\pm$ 0.016 & 7.679 $\pm$ 2.748  & \multicolumn{1}{|c|}{0.627 $\pm$ 0.013}  & 74.32 $\pm$ 2.70 & 68.22 $\pm$ 1.89 & 74.92 $\pm$ 4.42 & 72.40 $\pm$ 1.01 & 61.76 $\pm$ 1.11 \\ 

\hline
\multirow{2}{*}{\STAB{\rotatebox[origin=c]{90}{AB-S}}} & NAD-GCL-FIX                   & 1.392 $\pm$ 0.065  & 0.952 $\pm$ 0.024 & 5.840 $\pm$ 0.877  & \multicolumn{1}{|c|}{0.609 $\pm$ 0.010}  & 73.60 $\pm$ 2.73 & 66.12 $\pm$ 1.80 & 73.32 $\pm$ 3.66 & 71.65 $\pm$ 0.94 & 60.41 $\pm$ 1.48 \\ 

&NAD-GCL-OPT             & 1.242 $\pm$ 0.096  & 0.897 $\pm$ 0.022 & 5.840 $\pm$ 0.877  & \multicolumn{1}{|c|}{0.609 $\pm$ 0.010}  & 73.69 $\pm$ 3.67 & 67.70 $\pm$ 1.78 & 74.40 $\pm$ 4.92 & 71.65 $\pm$ 0.94 & 61.14 $\pm$ 1.43 \\ 

\hline\hline
\multirow{2}{*}{\STAB{\rotatebox[origin=c]{90}{Ours}}}&AD-GCL-FIX     & \textbf{1.217 $\pm$ 0.087}  & \textbf{0.842 $\pm$ 0.028}$^\star$ & \textbf{5.150 $\pm$ 0.624}$^\star$  & \multicolumn{1}{|c|}{\textbf{0.578 $\pm$ 0.012}$^\star$} & \textbf{76.37 $\pm$ 2.03} & 68.24 $\pm$ 1.47 & \textbf{80.77 $\pm$ 3.92} & 71.42 $\pm$ 0.73 & \textbf{63.19 $\pm$ 0.95} \\

&AD-GCL-OPT & \textbf{1.136 $\pm$ 0.050}$^\star$  & \textbf{0.812 $\pm$ 0.020}$^\star$ & \textbf{4.145 $\pm$ 0.369}$^\star$ & \multicolumn{1}{|c|}{\textbf{0.544 $\pm$ 0.004}$^\star$}  & \textbf{77.27 $\pm$ 2.56} & \textbf{69.54 $\pm$ 1.92} & \textbf{80.77 $\pm$ 3.92} & \textbf{72.92 $\pm$ 0.86} & \textbf{63.19 $\pm$ 0.95} \\ \hline \hline
\end{tabular}%
}

\caption{\small{Unsupervised learning performance for (TOP) biochemical and social network classification in TU datasets~\cite{Morris+2020} (Averaged accuracy $\pm$ std. over 10 runs) and (BOTTOM) chemical molecules property prediction in OGB datasets~\cite{hu2020open} (mean $\pm$ std. over  10 runs). %
\textbf{Bold}/\textbf{Bold}$^\star$ indicats our methods outperform baselines with $\geq$ 0.5/$\geq$ 2 std respectively. Fully supervised (F-GIN) results are shown \textbf{only} for placing GRL methods in perspective. Ablation-study (AB-S) results do not count as baselines.}}
\label{tab:unsupervised_learning_ogbg}
\vspace{-5mm}
\end{table}

Tables~\ref{tab:unsupervised_learning_ogbg} 
show the results for unsupervised graph level property prediction in social and chemical domains respectively.  
We witness the big performance gain of AD-GCL as opposed to all baselines across all the datasets. Note GraphCL utilizes extensive evaluation to select the best combination of augmentions over a broad GDA family including node-dropping, edge dropping and subgraph sampling. Our results indicate that such extensive evaluation may not be necessary while optimizing  the augmentation strategy in an adversarial way is greatly beneficial.

We stress that edge dropping is not cherry picked as the search space of augmentation strategies. Other search spaces may even achieve better performance, while an extensive investigation is left for the future work. 

Moreover, AD-GCL also clearly improves upon the performance against its non-adversarial counterparts (NAD-GCL) across all the datasets, which further demonstrates stable and significant advantages of the AD-GCL principle. Essentially, the input-graph-dependent augmentation learnt by AD-GCL yields much benefit.
Finally, we compare AD-GCL-FIX with AD-GCL-OPT. 
Interestingly, two methods achieve comparable results though AD-GCL-OPT is sometimes better. This observation implies that the AD-GCL principle may be robust to the choice of $\lambda_{\text{reg}}$ and thus motivates the analysis in the next subsection. Moreover, weak information from the downstream tasks indeed help with controlling the search space and further betters the performance. We also list the optimal $\lambda_{\text{reg}}$'s of AD-GCL-OPT for different datasets in Appendix~\ref{apd:more_results_opt_reg} for the purpose of comparison and reproduction.
\subsubsection{Note on the linear downstream classifier}
We find that the choice of the downstream classifier can significantly affect the evaluation of the self-supervised representations. InfoGraph~\cite{sun2019infograph} and GraphCL~\cite{you2020graph} adopt a non-linear SVM model as the downstream classifier. Such a non-linear model is more powerful than the linear model we adopt and thus causes some performance gap between the results showed in Table~\ref{tab:unsupervised_learning_ogbg} (TOP) and (BOTTOM) and their original results (listed in Appendix~\ref{apd:unsup_nonlinear} as Table~\ref{tab:unsup_tu_non_linear_eval_results}). We argue that using a non-linear SVM model as the downstream classifier is unfair, because the performance of even a randomly initialized untrained GIN (RU-GIN) is significantly improved  (comparing results from Table~\ref{tab:unsupervised_learning_ogbg} (TOP) to Table~\ref{tab:unsup_tu_non_linear_eval_results} ). Therefore, we argue for adopting a linear classifier protocol as suggested by~\cite{tschannen2019mutual}.
That having been said, our methods (both AD-GCL-FIX and AD-GCL-OPT) still performs significantly better than baselines in most cases, even when a non-linear SVM classifer is adopted, as shown in Table~\ref{tab:unsup_tu_non_linear_eval_results}. Several relative gains are there no matter whether the downstream classifier is a simple linear model (Tables~\ref{tab:unsupervised_learning_ogbg}) or a non-linear SVM model (Table~\ref{tab:unsup_tu_non_linear_eval_results}). AD-GCL methods significantly outperform InfoGraph in 5 over 8 datasets and GraphCL in 6 over 8 datasets. This further provides the evidence for the effectiveness of our method. Details on the practical benefits of linear downstream models can be found in Appendix~\ref{apd:unsup_nonlinear}.

\vspace{-3mm}
\subsection{Analysis of Regularizing the GDA Model}
\vspace{-2mm}
\label{exp:reg_analysis}

\begin{figure}[t]
    \centering
    \includegraphics[width=\textwidth]{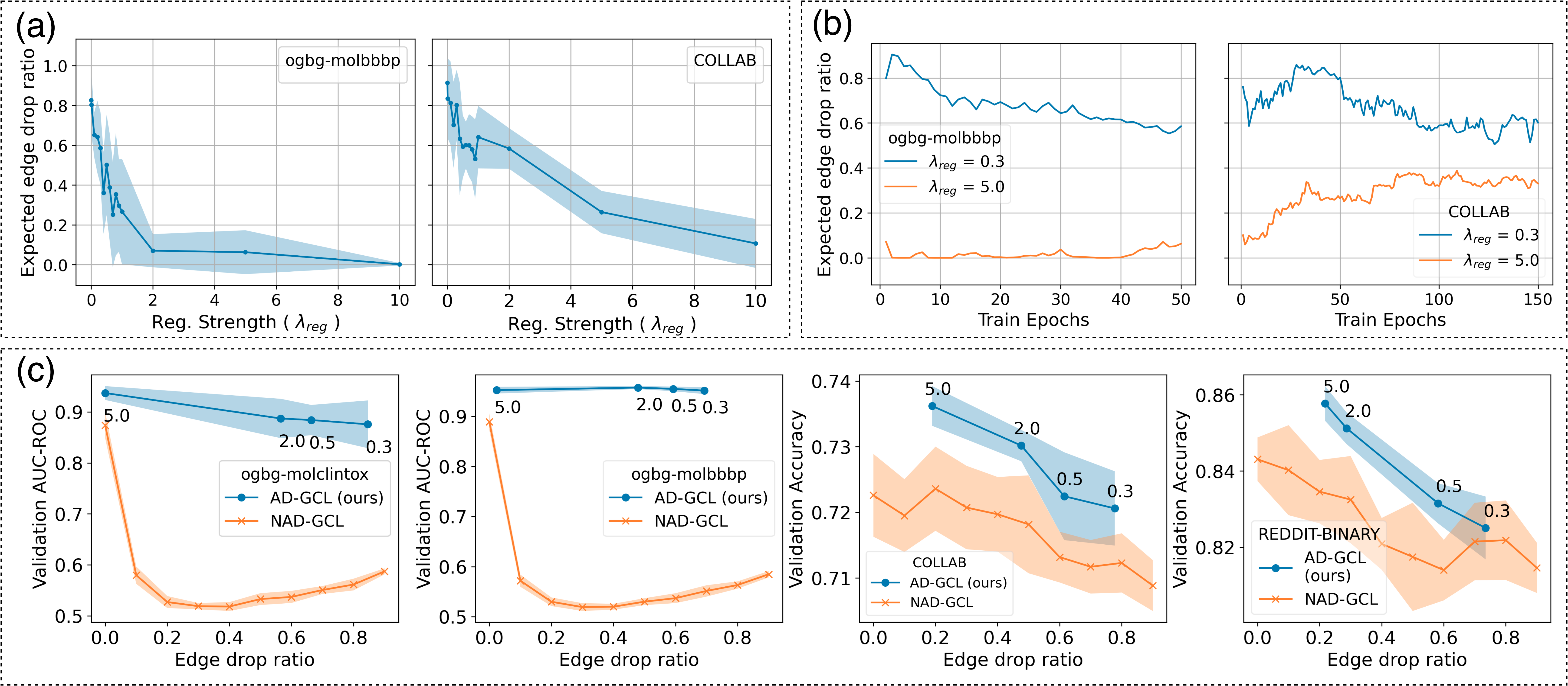}
    \caption{\small{(a) $\lambda_{\text{reg}}$ \textit{v.s.} expected edge drop ratio $\mathbb{E}_\mathcal{G}[\sum_{e} \omega_e/|E|]$ (measured at saddle point of Eq.\ref{eq:ad-gcl-encoder-augmentor-reg}). (b) Training dynamics of expected drop ratio for $\lambda_{\text{reg}}$. (c) Validation performance for graph classification \textit{v.s.} edge drop ratio. Compare AD-GCL and GCL with non-adversarial edge dropping. The markers on AD-GCL's performance curves show the $\lambda_{\text{reg}}$ used.}}
    \label{fig:reg_term_analysis}
    \vspace{-4mm}
\end{figure}

Here, we study how different $\lambda_{\text{reg}}$'s impact the expected edge drop ratio of AD-GCL at the saddle point of Eq.\ref{eq:ad-gcl-encoder-augmentor-reg} and further impact the model performance on the validation datasets. Due to the page limitation, we focus on classification tasks in the main text while leaving the discussion on regression tasks in the Appendix~\ref{apd:more_results_reg_analysis}. Figure~\ref{fig:reg_term_analysis} 
shows the results.

As shown in Figure~\ref{fig:reg_term_analysis}(a), a large $\lambda_{\text{reg}}$ tends to yield a small expected edge drop ratio at the convergent point, which matches our expectation. $\lambda_{\text{reg}}$ ranging from 0.1 to 10.0 corresponds to dropping almost everything (80\% edges) to nothing (<10\% edges). The validation performance in 
Figure~\ref{fig:reg_term_analysis}(c) is out of our expectation. We find that for classification tasks, the performance of the encoder is extremely robust to different choices of $\lambda_{\text{reg}}$'s when trained w.r.t. the AD-GCL principle, though the edge drop ratios at the saddle point are very different. However, the non-adversarial counterpart NAD-GCL is sensitive to different edge drop ratios, especially on the molecule dataset (e.g., ogbg-molclitox, ogbg-molbbbp). We actually observe the similar issue of NAD-GCL across all molecule datasets (See Appendix~\ref{apd:more_results_perf_vs_edge_drop}). More interesting aspects of our results appear at the extreme cases. When $\lambda_{\text{reg}}\geq 5.0$, the convergent edge drop ratio is close to 0, which means no edge dropping, but AD-GCL still significantly outperforms naive GCL with small edge drop ratio. When $\lambda_{\text{reg}}=0.3$, the convergent edge drop ratio is greater than 0.6, which means dropping more than half of the edges, but AD-GCL still keeps reasonable performance. We suspect that such benefit comes from the training dynamics of AD-GCL (examples as shown in Figure~\ref{fig:reg_term_analysis}(b)). Particularly, optimizing augmentations allows for non-uniform edge-dropping probability. During the optimization procedure, AD-GCL pushes high drop probability on redundant edges while low drop probability on critical edges, which allows the encoder to differentiate redundant and critical information. This cannot be fully explained by the final convergent edge drop ratio 
 and motivates future investigation of AD-GCL from a more in-depth theoretical perspective. 
\vspace{-3mm}
\subsection{Transfer Learning}
\vspace{-2mm}
\label{exp:transfer}
\begin{table}[t]
\centering
\renewcommand{\arraystretch}{1.25}
\newcommand{\STAB}[1]{\begin{tabular}{@{}c@{}}#1\end{tabular}}
\resizebox{\textwidth}{!}{%
\begin{tabular}{lcccccccc|c}
\hline
Pre-Train Dataset   & \multicolumn{8}{c|}{\begin{tabular}[c]{@{}c@{}} ZINC 2M\end{tabular}}                                                                                                                              & \begin{tabular}[c]{@{}c@{}}PPI-306K\end{tabular} \\ \hline
Fine-Tune Dataset    & BBBP                                & Tox21                                & SIDER                                & ClinTox                              & BACE                                 & \multicolumn{1}{c}{HIV} & MUV              & ToxCast & PPI                                                                                       \\ \hline
No Pre-Train & 65.8 $\pm$ 4.5                      & 74.0 $\pm$ 0.8                       & 57.3 $\pm$ 1.6                       & 58.0 $\pm$ 4.4                       & 70.1 $\pm$ 5.4                       & 75.3 $\pm$ 1.9          & 71.8 $\pm$ 2.5   &  63.4 $\pm$ 0.6 & 64.8 $\pm$ 1.0                                                                            \\
EdgePred~\cite{hu2019strategies}     & 67.3 $\pm$ 2.4                      & 76.0 $\pm$ 0.6                       & 60.4 $\pm$ 0.7                       & 64.1 $\pm$ 3.7                       & 79.9 $\pm$ 0.9                       & 76.3 $\pm$ 1.0          & 74.1 $\pm$ 2.1   & 64.1$\pm $ 0.6 & 65.7 $\pm$ 1.3                                                                            \\
AttrMasking~\cite{hu2019strategies}  & 64.3 $\pm$ 2.8                      & 76.7 $\pm$ 0.4                       & 61.0 $\pm$ 0.7                       & 71.8 $\pm$ 4.1                       & 79.3 $\pm$ 1.6                       & 77.2 $\pm$ 1.1          & 74.7 $\pm$ 1.4   & 64.2 $\pm$ 0.5 & 65.2 $\pm$ 1.6                                                                            \\
ContextPred~\cite{hu2019strategies}  & 68.0 $\pm$ 2.0                      & 75.7 $\pm$ 0.7                       & 60.9 $\pm$ 0.6                       & 65.9 $\pm$ 3.8                       & 79.6 $\pm$ 1.2                       & 77.3 $\pm$ 1.0          & 75.8 $\pm$ 1.7   & 63.9 $\pm$ 0.6& 64.4 $\pm$ 1.3                                                                            \\
InfoGraph~\cite{sun2019infograph}    & 68.8 $\pm$ 0.8                      & 75.3 $\pm$ 0.5                       & 58.4 $\pm$ 0.8                       & 69.9 $\pm$3.0                        & 75.9 $\pm$ 1.6                       & 76.0 $\pm$ 0.7          & 75.3 $\pm$ 2.5   & 62.7 $\pm$ 0.4 & 64.1 $\pm$ 1.5                                                                            \\
GraphCL~\cite{you2020graph}      & 69.68 $\pm$ 0.67                    & 73.87 $\pm$ 0.66                     & 60.53 $\pm$ 0.88                     & 75.99 $\pm$ 2.65                     & 75.38 $\pm$ 1.44                     & 78.47 $\pm$ 1.22        & 69.8 $\pm$ 2.66  &  62.40 $\pm$ 0.57 & 67.88 $\pm$ 0.85                                                                          \\ \hline\hline
AD-GCL-FIX         & \multicolumn{1}{l}{70.01 $\pm$1.07} & \multicolumn{1}{l}{76.54 $\pm$ 0.82} & \multicolumn{1}{l}{\textbf{63.28 $\pm$ 0.79}} & \multicolumn{1}{l}{\textbf{79.78 $\pm$ 3.52}} & \multicolumn{1}{l}{78.51 $\pm$ 0.80} & 78.28 $\pm$ 0.97        & 72.30 $\pm$ 1.61 & 63.07 $\pm$ 0.72 & \textbf{68.83 $\pm$ 1.26} \\
\centering{Our Ranks} & 1 & 2 & 1 & 1 & 4 & 2 & 5 & 5 &  1 \\
\hline\hline
\end{tabular}%
}
\caption{\small{Transfer learning performance for chemical molecules property prediction (mean ROC-AUC $\pm$ std. over  10 runs). \textbf{Bold} indicates our methods outperform baselines with $\geq$ 0.5 std..} }
\vspace{-5mm}
\label{tab:transfer_learning}
\end{table}
Next, we evaluate the GNN encoders trained by AD-GCL on transfer learning to predict chemical molecule properties and biological protein functions. We follow the setting in~\cite{hu2019strategies} and use the same datasets: GNNs are pre-trained on one dataset using self-supervised learning and later fine-tuned on another dataset to test out-of-distribution performance. Here, we only consider AD-GCL-FIX as AD-GCL-OPT is only expected to have better performance. We adopt baselines including no pre-trained GIN (\textit{i.e.,} without self-supervised training on the first dataset and with only fine-tuning), InfoGraph~\cite{sun2019infograph}, GraphCL~\cite{you2020graph}, three different pre-train strategies in~\cite{hu2019strategies} including edge prediction, node attribute masking and context prediction that utilize edge, node and subgraph context respectively. More detailed setup is given in Appendix~\ref{apd:exp_settings}. 

According to Table~\ref{tab:transfer_learning}, AD-GCL-FIX significantly outperforms baselines in 3 out of 9 datasets and achieves a mean rank of 2.4 across these 9 datasets which is better than all baselines. Note that although AD-GCL only achieves 5th on some datasets, AD-GCL still significantly outperforms InfoGraph~\cite{sun2019infograph} and GraphCL~\cite{you2020graph}, both of which are strong GNN self-training baselines. In contrast to InfoGraph~\cite{sun2019infograph} and GraphCL~\cite{you2020graph}, AD-GCL achieves some performance much closer to those baselines (EdgePred, AttrMasking and ContextPred) based on domain knowledge and extensive evaluation in~\cite{hu2019strategies}. This is rather significant as our method utilizes only edge dropping GDA, which again shows the effectiveness of the AD-GCL principle.

\vspace{-2mm}
\subsection{Semi-Supervised Learning}
\vspace{-1mm}
\label{exp:semisup}
\begin{table}[t]
\centering
\renewcommand{\arraystretch}{1.25}
\resizebox{0.8\textwidth}{!}{%
\begin{tabular}{lccc|ccc}
\hline
Dataset            & NCI1             & PROTEINS         & DD               & COLLAB           & RDT-B            & RDT-M5K          \\ \hline
No Pre-Train      & 73.72 $\pm$ 0.24 & 70.40 $\pm$ 1.54 & 73.56 $\pm$ 0.41 & 73.71$\pm$ 0.27  & 86.63 $\pm$ 0.27 & 51.33 $\pm$ 0.44 \\
SS-GCN-A & 73.59 $\pm$ 0.32 & 70.29 $\pm$ 0.64 & 74.30 $\pm$ 0.81 & 74.19 $\pm$ 0.13 & 87.74 $\pm$ 0.39 & 52.01 $\pm$ 0.20 \\
GAE~\cite{kipf2016variational}           & 74.36 $\pm$ 0.24 & 70.51 $\pm$ 0.17 & 74.54 $\pm$ 0.68 & 75.09 $\pm$ 0.19 & 87.69 $\pm$ 0.40 & 53.58 $\pm$ 0.13 \\
InfoGraph~\cite{sun2019infograph}     & 74.86 $\pm$ 0.26 & 72.27 $\pm$ 0.40 & 75.78 $\pm$ 0.34 & 73.76 $\pm$ 0.29 & 88.66 $\pm$ 0.95 & 53.61 $\pm$ 0.31 \\
GraphCL~\cite{you2020graph}      & 74.63 $\pm$ 0.25 & 74.17 $\pm$ 0.34 & 76.17 $\pm$ 1.37 & 74.23 $\pm$ 0.21 & 89.11 $\pm$ 0.19 & 52.55 $\pm$ 0.45 \\ \hline
AD-GCL-FIX        & \textbf{75.18 $\pm$ 0.31} &  73.96 $\pm$ 0.47 & \textbf{77.91 $\pm$ 0.73}$^\star$ & \textbf{75.82 $\pm$ 0.26}$^\star$ & \textbf{90.10 $\pm$ 0.15}$^\star$ & 53.49 $\pm$ 0.28 \\ 
Our Ranks & 1 &  2  & 1 & 1 & 1  &  3 \\
\hline
\end{tabular}%
}
\caption{\small{Semi-supervised learning performance with 10\% labels on TU datasets~\cite{Morris+2020} (10-Fold Accuracy (\%)$\pm$ std over 5 runs). \textbf{Bold}/\textbf{Bold}$^\star$ indicate our methods outperform baselines with $\geq$ 0.5 std/ $\geq$ 2 std respectively.} } 
\label{tab:semi_supervised_learning}
\vspace{-5mm}
\end{table}
Lastly, we evaluate AD-GCL on semi-supervised learning for graph classification on the benchmark TU datasets~\cite{Morris+2020}. We follow the setting in \cite{you2020graph}:  GNNs are pre-trained on one dataset using self-supervised learning and later fine-tuned based on 10\% label supervision on the same dataset.  Again, we only consider AD-GCL-FIX and compare it with several baselines in \cite{you2020graph}: 1) no pre-trained GCN, which is directly trained by the 10\% labels from scratch, 2) SS-GCN-A, a baseline that introduces more labelled data by creating random augmentations and then gets trained from scratch, 3) a predictive method GAE~\cite{kipf2016variational} that utilizes adjacency reconstruction in the pre-training phase, and GCL methods, 4) InfoGraph~\cite{sun2019infograph} and 5) GraphCL~\cite{you2020graph}. Note that here we have to keep the encoder architecture same and thus  AD-GCL-FIX adopts GCN as the encoder. Table~\ref{tab:semi_supervised_learning} shows the results. AD-GCL-FIX significantly outperforms baselines in 3 out of 6 datasets and achieves a mean rank of 1.5 across these 6 datasets, which again demonstrates the strength of AD-GCL. 
\vspace{-3mm}
\section{Conclusions}
\vspace{-2mm}
    In this work we have developed a theoretically motivated, novel principle: \textit{AD-GCL} that goes a step beyond the conventional InfoMax objective for self-supervised learning of GNNs. The optimal GNN encoders that are agnostic to the downstream tasks are the ones that capture the minimal sufficient information to identify each graph in the dataset. To achieve this goal, AD-GCL suggests to better graph contrastive learning via optimizing graph augmentations in an adversarial way.  Following this principle, we developed a practical instantiation based on learnable edge dropping. We have extensively analyzed and demonstrated the benefits of AD-GCL and its instantiation with real-world datasets for graph property prediction in unsupervised, transfer and semi-supervised learning settings.

\begin{ack}
We greatly thank the actionable suggestions given by reviewers and the area chair.
S.S. and J.N. are supported by the National Science Foundation under contract numbers CCF-1918483 and IIS-1618690. P.L. is partly supported by the 2021 JP Morgan Faculty Award and the National Science Foundation (NSF) award HDR-2117997.
\end{ack}



\bibliographystyle{ACM.bst}
\bibliography{neurips_2021}


\newpage
\appendix
\section{Summary of the Appendix}

In the appendix, we provide the detailed proof of the Theorem~\ref{thm:main} (Sec.~\ref{apd:prf}), a review of WL tests (Sec.~\ref{apd:1wl}), the detailed algorithmic format of our instantiation of AD-GCL (Sec.~\ref{apd:algo}), the summary of datasets (Sec.~\ref{apd:datasets}), more regularization hyperparameter analysis (Sec.~\ref{apd:more_results}), detailed experimental settings and complete evaluation results (Sec.~\ref{apd:exp_settings}), computing resources (Sec.~\ref{apd:compute_resource}) and discussion on broader impacts (Sec.\ref{apd:broader_impact}).

\section{Proof of Theorem~\ref{thm:main}}
\label{apd:prf}
We repeat Theorem~\ref{thm:main} as follows.
\begin{theorem} 
Suppose the encoder $f$ is implemented by a GNN as powerful as the 1-WL test. Suppose $\mathcal{G}$ is a countable space and thus $\mathcal{G'}$ is a countable space. Then, the optimal solution $(f^*, T^*)$ to AD-GCL satisfies, letting $T'^{*}(G') = \mathbb{E}_{G\sim \mathbb{P}_{\mathcal{G}}}[T^*(G)| G\cong G']$, \vspace{-1mm}

\begin{enumerate}[leftmargin=*]
    \item $I(f^*(t^*(G)); G\,|\,Y) \leq  \min_{T\in\mathcal{T}} I(t'(G'); G') - I(t'^*(G'); Y) $, where $t^*(G)\sim T^*(G)$, $t'(G')\sim T'(G')$, $t'^*(G')\sim T'^{*}(G')$, $(G,Y)\sim \mathbb{P}_{\mathcal{G}\times \mathcal{Y}}$ and $(G',Y)\sim \mathbb{P}_{\mathcal{G}'\times \mathcal{Y}}$.
    \item $I(f^*(G); Y)\geq I(f^*(t'^*(G')); Y) = I(t'^*(G'); Y)$, where $t^*(G)\sim T^*(G)$, $t'^*(G')\sim T'^*(G')$, $(G,Y)\sim \mathbb{P}_{\mathcal{G}\times \mathcal{Y}}$ and $(G',Y)\sim \mathbb{P}_{\mathcal{G}'\times \mathcal{Y}}$.
\end{enumerate}
\vspace{-1mm}
\end{theorem}
\begin{proof}
Because $\mathcal{G}$ and $\mathcal{G'}$ are countable, $P_{\mathcal{G}}$ and $P_{\mathcal{G}'}$ are defined over countable sets and thus discrete distribution. Later we may call a function $z(\cdot)$ can distinguish two graphs $G_1$, $G_2$ if $z(G_1)\neq z(G_2)$.

Moreover, for notational simplicity, we consider the following definition. Because $f^*$ is as powerful as the 1-WL test. Then, for any two graphs $G_1, G_2 \in \mathcal{G}$, $G_1\cong G_2$, $f^*(G_1) = f^*(G_2)$. We may define a mapping over $\mathcal{G}'$, also denoted by $f^*$ which simply satisfies  $f^*(G'): \triangleq f^*(G)$, where $G\cong G'$, and $G\in\mathcal{G}$ and $G'\in \mathcal{G}'$.

We first prove the statement 1, \textit{i.e.,} the upper bound. We have the following inequality: Recall that $T'^{*}(G') = \mathbb{E}_{G\sim \mathbb{P}_{\mathcal{G}}}[T^*(G)| G\cong G']$ and $t'^*(G')\sim T'^{*}(G')$.
\begin{align}
    I(t'^*(G');G') & = I(t'^*(G'); (G', Y)) - I(t'^*(G');Y|G') ]\nonumber \\
    &\stackrel{(a)}{=} I( t'^*(G'); (G', Y)) \nonumber\\
    &= I( t'^*(G'); Y) + I( t'^*(G'); G'|Y) \nonumber\\
    & \stackrel{(b)}{\geq} I( f^*(t'^*(G')); G'|Y) + I(t'^*(G');Y) \label{prf:eq1}
\end{align}
where $(a)$ is because $t'^*(G')\perp_{G'} Y$, $(b)$ is because the data processing inequality~\cite{cover1999elements}. Moreover, because $f^*$ could be as powerful as the 1-WL test and thus could be injective in $\mathcal{G}'$ a.e. with respect to the measure $\mathbb{P}_{\mathcal{G}'}$. Then, for any GDA $T(\cdot)$ and $T'(G')=\mathbb{E}_{G\sim \mathbb{P}_{\mathcal{G}}}[T(G)| G\cong G']$,
\begin{align}\label{prf:eq1p5}
    I(t'(G');G') = I(f^*(t'(G'));f^*(G')) = I(f^*(t(G));f^*(G)), 
\end{align}
where $t'(G')\sim T'(G'),\,t(G)\sim T(G)$. Here, the second equality is because $f^*(G) = f^*(G')$ and $T'(G')=\mathbb{E}_{G\sim \mathbb{P}_{\mathcal{G}}}[T(G)| G\cong G']$.

Since $T^*= \argmin_{T\in \mathcal{T}} I(f(t^*(G));f(G))$ where $t^*(G)\sim T^*(G)$ and Eq.\ref{prf:eq1p5}, we have 
\begin{align}
I(t'^*(G');G') = \argmin_{T\in \mathcal{T}} I(t'(G');G'),\, \text{where } t'(G')\sim T'(G')= \mathbb{E}_{G\sim \mathbb{P}_{\mathcal{G}}}[T(G)| G\cong G']. \label{prf:eq2}
\end{align}

Again, because by definition $f^* = \argmax_f I(f(G); f(t^*(G)))$ and $f^*$ could be as powerful as the 1-WL test, its counterpart defined over $\mathcal{G}'$, i.e., $f^\star$, must be injective over $\mathcal{G'}\cap \text{Supp}(\mathbb{E}_{G'\sim \mathbb{P}_{\mathcal{G}'}}[T’^*(G')])$ a.e. with respect to the measure $\mathbb{P}_{\mathcal{G}'}$ to achieve such mutual information maximization. Here, Supp($\mu$) defines the set where $\mu$ has non-zero measure. Because of the definition of $T'^*(G')=\mathbb{E}_{G\sim \mathbb{P}_{\mathcal{G}}}[T^*(G)| G\cong G']$, $$\mathcal{G'}\cap  \text{Supp}(\mathbb{E}_{G'\sim \mathbb{P}_{\mathcal{G}'}}[T’^*(G')]) = \mathcal{G'}\cap  \text{Supp}(\mathbb{E}_{G\sim \mathbb{P}_{\mathcal{G}}}[T^*(G)]).$$ 

Therefore, $f^*$ is a.e. injective over $\mathcal{G'}\cap \text{Supp}(\mathbb{E}_{G\sim \mathbb{P}_{\mathcal{G}}}[T^*(G)])$ and thus
\begin{align} \label{prf:eq3}
I( f^*(t'^*(G')); G'|Y) = I( f^*(t^*(G)); G'|Y), 
\end{align}
Moreover, as $f^*$ cannot  cannot distinguish more graphs in $\mathcal{G}$ than $\mathcal{G'}$ as the power of $f^*$ is limited by 1-WL test, thus,
\begin{align} \label{prf:eq3p5}
I( f^*(t^*(G)); G'|Y) = I( f^*(t^*(G)); G|Y). 
\end{align}
Plugging Eqs.\ref{prf:eq2},\ref{prf:eq3},\ref{prf:eq3p5} into Eq.\ref{prf:eq1}, we achieve
\begin{align*}
    I( f^*(t^*(G)); G|Y) \leq \argmin_{T\in \mathcal{T}} I(t'(G');G') - I(t'^*(G');Y)
\end{align*}
where $t'(G')\sim T'(G')= \mathbb{E}_{G\sim \mathbb{P}_{\mathcal{G}}}[T(G)| G\cong G']$ and $t'^*(G') \sim T'^*(G')= \mathbb{E}_{G\sim \mathbb{P}_{\mathcal{G}}}[T^*(G)| G\cong G']$, which gives us the statement 1, which is the upper bound.

We next prove the statement 2, \textit{i.e.,} the lower bound. Recall $(T^*, f^*)$ is the optimal solution to Eq.\ref{eq:ad-gcl} and $t^*(\cdot)$ denotes samples from $T^*(\cdot)$.

Again, because $f^* = \argmax_f I(f(G); f(t^*(G)))$, $f^*$ must be injective on $\mathcal{G'}\cap \text{Supp}(\mathbb{E}_{G'\sim \mathbb{P}_{\mathcal{G}'}}[T’^*(G')])$ a.e. with respect to the measure $\mathbb{P}_{\mathcal{G'}}$. Given $t'^*(G')$, $t'^*(G')\rightarrow f^*(t'^*(G'))$ is an injective deterministic mapping. Therefore, for any random variable $Q$,
\begin{align*}
    I(f^*(t'^*(G')); Q) = I(t'^*(G'); Q), \quad \text{where } G'\sim \mathbb{P_{\mathcal{G}'}}, t'^*(G')\sim T'^*(G').
\end{align*}
Of course, we may set $Q=Y$. So,
\begin{align}\label{prf:eq4}
    I(f^*(t'^*(G')); Y) = I(t'^*(G'); Y), \quad \text{where } (G',Y)\sim \mathbb{P_{\mathcal{G}'\times\mathcal{Y}}}, t'^*(G')\sim T'^*(G').
\end{align}
Because of the data processing inequality~\cite{cover1999elements} and $T'^*(G')=\mathbb{E}_{G\sim \mathbb{P}_{\mathcal{G}}}[T^*(G)| G\cong G']$, we further have 
\begin{align}\label{prf:eq5}
    I(f^*(t^*(G)); Y) \geq I(f^*(t'^*(G')); Y) ,
\end{align}
\text{where } $(G',Y)\sim \mathbb{P_{\mathcal{G}'\times\mathcal{Y}}}, (G,Y)\sim \mathbb{P_{\mathcal{G}\times\mathcal{Y}}}, t'^*(G')\sim T'^*(G'), t^*(G)\sim T^*(G).$

Further because of the data processing inequality~\cite{cover1999elements}, 
\begin{align}\label{prf:eq6}
    I(f^*(G); Y) \geq I(f^*(t^*(G)); Y). 
\end{align}
Combining Eqs.\ref{prf:eq4}, \ref{prf:eq5}, \ref{prf:eq6}, we have
\begin{align*}
     I(f^*(G); Y) \geq I(f^*(t^*(G)); Y) \geq I(f^*(t'^*(G')); Y) =  I(t'^*(G'); Y),
\end{align*}
which concludes the proof of the lower bound.

\end{proof}

\section{A Brief Review of the Weisfeiler-Lehman (WL) Test}
\label{apd:1wl}
Two graphs $G_1$ and $G_2$ are called to be isomorphic if there is a mapping between the nodes of the graphs such that their adjacencies are preserved. For a general class of graphs, without the knowledge of the mapping, determining if $G_1$ and $G_2$ are indeed isomorphic is challenging and there has been no known polynomial time algorithms utill now~\cite{babai2018groups}. The best algorithm till now has complexity $2^{O(\log n)^3}$ where $n$ is the size of the graphs of interest~\cite{helfgott2017graph}. 

The family of Weisfeiler-Lehman tests \cite{weisfeiler1968reduction} (specifically the 1-WL test) offers a very efficient way perform graph isomorphism testing by generating canonical forms of graphs. Specifically, the 1-WL test follows an iterative color refinement algorithm. Let, graph $G = (V, E)$ and let $C: V \rightarrow \mathcal{C}$ denote a coloring function that assigns each vertex $v \in V$ a color $C_v$. Nodes with different features are associated with different colors. These colors constitute the initial colors $C_0$ of the algorithm i.e. $C_{0,v} = C_v$ for every vertex $v \in V$. Now, for each vertex $v$ and each iteration $i$, the algorithm creates a new set of colors from the color $C_{i-1,v}$ and the colors $C_{i-1,u}$ of every vertex $u$ that is adjacent to $v$. This multi-set of colors is then mapped to a new color (say using a unique hash). Basically, the color refinement follows 
\begin{align}\label{eq:wl-ref}
    C_{i,v} \leftarrow \text{Hash}(C_{i-1,v}, \{C_{i-1,u|u\in \mathcal{N}_v}\}),
\end{align}
where the above Hash function is an injective mapping. This iteration goes on until when the list of colors stabilises, i.e. at some iteration $N$, no new colors are created. The final set of colors serves as the the canonical form of a graph. 

Intuitively, if the canonical forms obtained by 1-WL test for two graphs are different, then the graphs are surely not isomorphic. But, it is possible for two non-isomorphic graphs to share a the same 1-WL canonical form. Though the 1-WL test can test most of the non-isomorphic graphs, it will fail in some corner cases. For example, it cannot distinguish regular graphs with the same node degrees and of the same sizes. 

As GNNs share the same iterative procedure as the 1-WL test by comparing Eq.~\ref{eq:wl-ref} and Eq.~\ref{eq:gnn_iteration}, GNNs are proved to be at most as powerful as the 1-WL test to distinguish isormorphic graphs~\cite{xu2018powerful,morris2019weisfeiler}. However, GNNs with proper design may achieve the power of the 1-WL test~\cite{xu2018powerful} and thus the assumption in Theorem~\ref{thm:main} is reasonable.

\section{The Training Algorithm for the Instantiation of AD-GCL}
\label{apd:algo}
Algorithm~\ref{algo:ad-gcl-train} describes the self-supervised training algorithm for AD-GCL with learnable edge-dropping GDA. Note that augmenter $T_{\Phi}(\cdot)$ with parameters $\Phi$ is implemented as a GNN followed by an MLP to obtain the Bernoulli weights $\omega_e$.

\makeatletter 
\g@addto@macro{\@algocf@init}{\SetKwInOut{paramH}{Hyper-Params}} 
\makeatother
\begin{algorithm}
\setstretch{1.25}
\caption{Training Learnable Edge-Dropping GDA under AD-GCL principle.}
\label{algo:ad-gcl-train}
    \KwIn{Data $\{G_m \sim \mathcal{G} \mid m = 1, 2 \dots M\}$;\newline
    Encoder $f_{\Theta}(\cdot)$;
    Augmenter $T_{\Phi}(\cdot)$; Projection Head $g_{\Psi}(\cdot)$; Cosine Similarity $sim(\cdot)$
    }
    \paramH{Edge-Dropping Regularization Strength $\lambda_{\text{reg}}$; learning rates $\alpha, \beta$}
    \KwOut{Trained Encoder $f_{\Theta}(\cdot)$}
    \Begin{
        \For{\text{number of training} epochs}{
            \For{sampled minibatch $\{G_n = (V_n, E_n) : n = 1, 2 \dots N\}$}{
                \For{n = 1 to N}{
                    $h_{1,n} = f_{\Theta}(G_n)$
                    
                    $z_{1,n}  = g_{\Psi}(h_{1,n})$ 
                    
                    $t(G_n) \sim T_{\Phi}(G_n)$
                    
                    \textbf{set} $p_{e}, \forall e \in E_n$ \textbf{from} $t(G_n)$
                    
                    $\mathcal{R}_n = \sum_{e  \in E_n} p_{e}/|E_n|$

                    $h_{2,n} = f_{\Theta}(t(G_n))$\;
                    
                    $z_{1,n}  = g_{\Psi}(h_{2,n})$ 
                    
                }
            \textbf{define} $\mathcal{L}_n = -\log \frac{\exp({sim(z_{1,n}, z_{2,n})})}{\sum_{n\prime = 1, n\prime \neq n}^{N}\exp{({sim(z_{1,n}, z_{2,n\prime})})}}$
            
            \tcc{calculate NCE loss for minibatch}
            
            $\mathcal{L} = \frac{1}{N} \sum_{n=1}^{N}\mathcal{L}_n$
            
            \tcc{calculate regularization term for minibatch}
            
            $\mathcal{R} = \frac{1}{N} \sum_{n=1}^{N}\mathcal{R}_n$
            
            \tcc{update augmenter params via gradient ascent}
            
            $\Phi \gets \Phi + \alpha \nabla_{\Phi} (\mathcal{L} - \lambda_{\text{reg}} * \mathcal{R})$ 
            
            \tcc{update enocder \& projection head \newline
            params via gradient descent}
            
            $\Theta \gets \Theta - \beta \nabla_{\Theta} (\mathcal{L})$;\quad$\Psi \gets \Psi - \beta \nabla_{\Psi} (\mathcal{L})$

            }

        }
        \Return{Encoder $f_{\Theta}(\cdot)$}
    }

\end{algorithm}
\section{Summary of Datasets}
\label{apd:datasets}
A wide variety of datasets from different domains for a range of graph property prediction tasks are used for our experiments. Here, we summarize and point out the specific experiment setting for which they are used.
\begin{itemize}
    \item Table~\ref{tab:dataset-stats-ogbg-unsup} shows the datasets for chemical molecular property prediction which are from Open Graph Benchmark (OGB)~\cite{hu2020open} and ZINC-10K~\cite{dwivedi2020benchmarking}. These are used in the unsupervised learning setting for both classification and regression tasks. We are the first one to considering using regression tasks and the corresponding datasets in the evaluation of self-supervised GNN.
    \item Table~\ref{tab:dataset-stats-tu-unsup} shows the datasets which contain biochemical and social networks. These are taken from the TU Benchmark Datasets~\cite{Morris+2020}. We use them for graph classification tasks in both unsupervised and semi-supervised learning settings.
    \item Table~\ref{tab:dataset-stats-transfer} shows the datasets consisting of biological interactions and chemical molecules from \cite{hu2019strategies}. These datasets are used for graph classification in the transfer learning setting.
\end{itemize}

\begin{table}[htb]
\centering
\resizebox{0.8\textwidth}{!}{%
\begin{tabular}{lcccccc@{}}
\toprule
Name             & \#Graphs & Avg \#Nodes & Avg \#Edges & \#Tasks & Task Type     & Metric  \\ \midrule
ogbg-molesol     & 1,128    & 13.3        & 13.7        & 1       & Regression    & RMSE    \\
ogbg-mollipo     & 4,200    & 27.0        & 29.5        & 1       & Regression    & RMSE    \\
ogbg-molfreesolv & 642      & 8.7         & 8.4         & 1       & Regression    & RMSE    \\
ogbg-molbace     & 1,513    & 34.1        & 36.9        & 1       & Binary Class. & ROC-AUC \\
ogbg-molbbbp     & 2,039    & 24.1        & 26.0        & 1       & Binary Class. & ROC-AUC \\
ogbg-molclintox  & 1,477    & 26.2        & 27.9        & 2       & Binary Class. & ROC-AUC \\
ogbg-moltox21    & 7,831    & 18.6        & 19.3        & 12      & Binary Class. & ROC-AUC \\
ogbg-molsider    & 1,427    & 33.6        & 35.4        & 27      & Binary Class. & ROC-AUC \\ 
ZINC-10K         & 12,000   & 23.16       & 49.83       & 1       & Regression    & MAE  \\\bottomrule
\end{tabular}%

}
\caption{Summary of chemical molecular properties datasets used for unsupervised learning experiments. Datasets obtained from OGB~\cite{hu2020open} and \cite{dwivedi2020benchmarking}}
\label{tab:dataset-stats-ogbg-unsup}
\end{table}
\begin{table}[htb]
\centering
\resizebox{0.6\textwidth}{!}{%
\begin{tabular}{lcccc@{}}
\toprule
Dataset  & \multicolumn{1}{l}{\#Graphs} & \multicolumn{1}{l}{Avg. \#Nodes} & \multicolumn{1}{l}{Avg. \#Edges} & \multicolumn{1}{l}{\#Classes}\\ \midrule
\multicolumn{5}{c}{Biochemical Molecules}                                                                                       \\ \midrule
NCI1        &   4,110   &   29.87   &    32.30    &  2   \\
PROTEINS    &   1,113   &   39.06   &   72.82    &  2  \\   
MUTAG       &   188     &   17.93   &    19.79  & 2    \\  
DD          &   1,178   &   284.32  &   715.66 & 2    \\  \midrule
\multicolumn{5}{c}{Social Networks}\\   \midrule
COLLAB          &   5,000   &   74.5    & 2457.78 & 3\\
REDDIT-BINARY   &   2,000   &   429.6   &  497.75 &  2\\
REDDIT-MULTI-5K &   4,999   &   508.8   &  594.87 & 5\\
IMDB-BINARY     &   1,000   &   19.8    &  96.53  & 2\\
IMDB-MULTI      &   1,500   &   13.0    &  65.94 & 3\\\bottomrule
\end{tabular}%
}
\caption{Summary of biochemical and social networks from TU Benchmark Dataset~\cite{Morris+2020} used for unsupervised and semi-supervised learning experiments. The evaluation metric for all these datasets is Accuracy.}
\label{tab:dataset-stats-tu-unsup}
\end{table}
\begin{table}[htb]
\centering
\resizebox{0.6\textwidth}{!}{%
\begin{tabular}{lcccc@{}}
\toprule
Dataset  & \multicolumn{1}{l}{Utilization} & \multicolumn{1}{l}{\#Graphs} & \multicolumn{1}{l}{Avg. \#Nodes} & \multicolumn{1}{l}{Avg. Degree} \\ \midrule
\multicolumn{5}{c}{Protein-Protein Interaction Networks}                                                                                       \\ \midrule
PPI-306K & Pre-Training                    & 306,925                      & 39.82                            & 729.62                          \\
PPI      & Finetuning                      & 88,000                       & 49.35                            & 890.77                          \\ \midrule
\multicolumn{5}{c}{Chemical Molecules}                                                                                                         \\ \midrule
ZINC-2M  & Pre-Training                    & 2,000,000                    & 26.62                            & 57.72                           \\
BBBP     & Finetuning                      & 2,039                        & 24.06                            & 51.90                           \\
Tox21    & Finetuning                      & 7,831                        & 18.57                            & 38.58                           \\
SIDER    & Finetuning                      & 1,427                        & 33.64                            & 70.71                           \\
ClinTox  & Finetuning                      & 1,477                        & 26.15                            & 55.76                           \\
BACE     & Finetuning                      & 1,513                        & 34.08                            & 73.71                           \\
HIV      & Finetuning                      & 41,127                       & 25.51                            & 54.93                           \\
MUV      & Finetuning                      & 93,087                       & 24.23                            & 52.55                           \\
ToxCast      & Finetuning                      & 8,576                       & 18.78                          & 38.52                           \\\bottomrule
\end{tabular}%
}
\caption{Summary of biological interaction and chemical molecule datasets from \cite{hu2019strategies}. Used for graph classification in transfer learning experiments. The evaluation metric is ROC-AUC.}
\label{tab:dataset-stats-transfer}
\end{table}

\section{Complete Results on Regularization Analysis}
\label{apd:more_results}

The main hyper-parameter for our method AD-GCL is the regularization strength $\lambda_{\text{reg}}$. Detailed sensitivity analysis is provided in Figures~\ref{fig:reg_term_analysis}, \ref{fig:appendix_perf_vs_edge_drop_regression} and \ref{fig:appendix_perf_vs_edge_drop_classification}. For the method AD-GCL-OPT, we tune $\lambda_{\text{reg}}$ over the validation set among $\{0.1, 0.3, 0.5, 1.0, 2.0, 5.0, 10.0\}$. For the ablation study, i.e. NAD-GCL-OPT the random edge drop ratio is tuned over the validation set among $\{0.0, 0.1, 0.2, 0.3, 0.4, 0.5, 0.6, 0.7, 0.8, 0.9\}$. 
 
\subsection{Optimal regularization strength values}
\label{apd:more_results_opt_reg}
\begin{table}[htb]
\centering
\resizebox{\textwidth}{!}{%
\begin{tabular}{@{}lccccccccc@{}}
\toprule
        & ogbg-molesol & ogbg-mollipo & ogbg-molfreesolv & ZINC-10K & ogbg-molbace & ogbg-molbbbp & ogbg-molclintox & ogbg-moltox21 & ogbg-molsider \\ \midrule
AD-GCL-OPT  &   0.4           &    0.1          &    0.3              &   0.8       &     10.0         &   10.0           &    5.0             &    10.0           &        5.0       \\ \midrule
        & COLLAB       & RDT-B        & RDT-M5K          & IMDB-B   & IMDB-M       & NCI1         & PROTEINS        & MUTAG         & DD            \\ \midrule
AD-GCL-OPT  &   5.0           &      5.0        &       10.0           &   2.0       &      10.0        &   5.0           &    1.0             &    10.0           &       10.0        \\ \bottomrule
\end{tabular}%
}
\caption{Optimal $\lambda_{\text{reg}}$ for AD-GCL on validation set that are used for reporting test performance in Tables~\ref{tab:unsupervised_learning_ogbg} (TOP) and (BOTTOM).}
\label{tab:opt_reg_vals}
\end{table}
Table~\ref{tab:opt_reg_vals} shows the optimal $\lambda_{\text{reg}}$ on the validation set that are used to report test performance in Tables~\ref{tab:unsupervised_learning_ogbg} (both TOP and BOTTOM).

\subsection{Effects of regularization on regression tasks}
\label{apd:more_results_reg_analysis}
Subplots in the topmost row of Figure~\ref{fig:appendix_perf_vs_edge_drop_regression} shows the validation performance for different $\lambda_{\text{reg}}$'s in AD-GCL and random edge drop ratios in NAD-GCL for regression tasks. These observations show an interesting phenomenon that is different from what we observe in classification tasks: for AD-GCL, small $\lambda_{\text{reg}}$ (which in-turn lead to large expected edge drop ratio) results in better performance. A similar trend can be observed even for NAD-GCL, where large random edge drop ratios results in better performance. However, AD-GCL is still uniformly better that NAD-GCL in that regard. We reason that, regression tasks (different from classification tasks) are more sensitive to node level information rather than structural fingerprints and thus, the edge dropping GDA family might not be the most apt GDA family. Modelling different learnable GDA families is left for future work and these observations motivate such steps.
\subsection{Effects of regularization on edge-drop ratio as complete results in Figure~\ref{fig:reg_term_analysis} setting.}
\label{apd:more_results_perf_vs_edge_drop}
\begin{figure}
    \centering
    \includegraphics[width=0.7\textwidth]{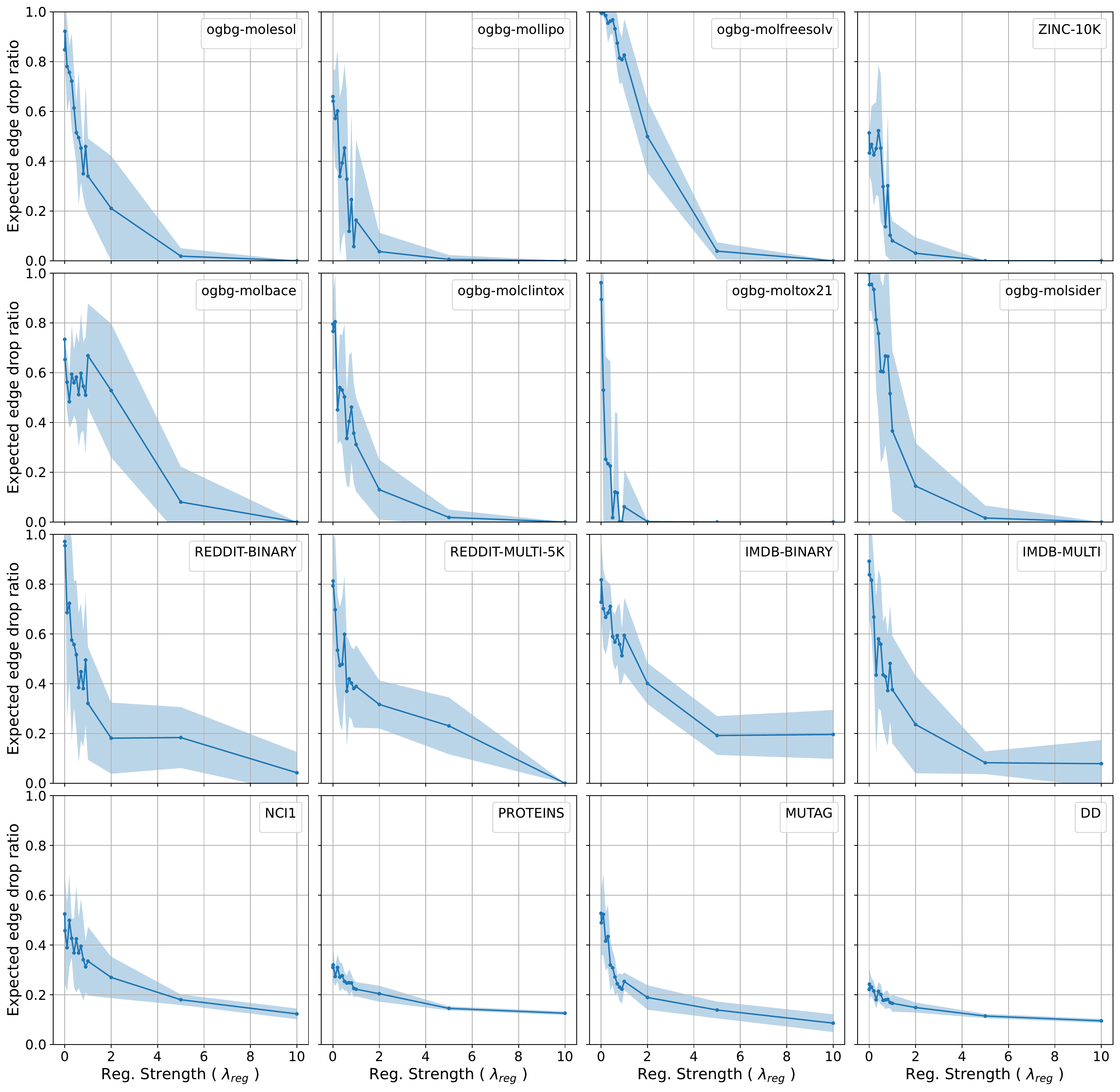}
    \caption{\small{$\lambda_{\text{reg}}$ \textit{v.s.} expected edge drop ratio $\mathbb{E}_\mathcal{G}[\sum_{e} \omega_e/|E|]$ (measured at saddle point of Eq.\ref{eq:ad-gcl-encoder-augmentor-reg}).}}
    \label{fig:appendix_lambda_vs_exp_edge_drop}
\end{figure}

\begin{figure}
    \centering
    \includegraphics[width=0.85\textwidth]{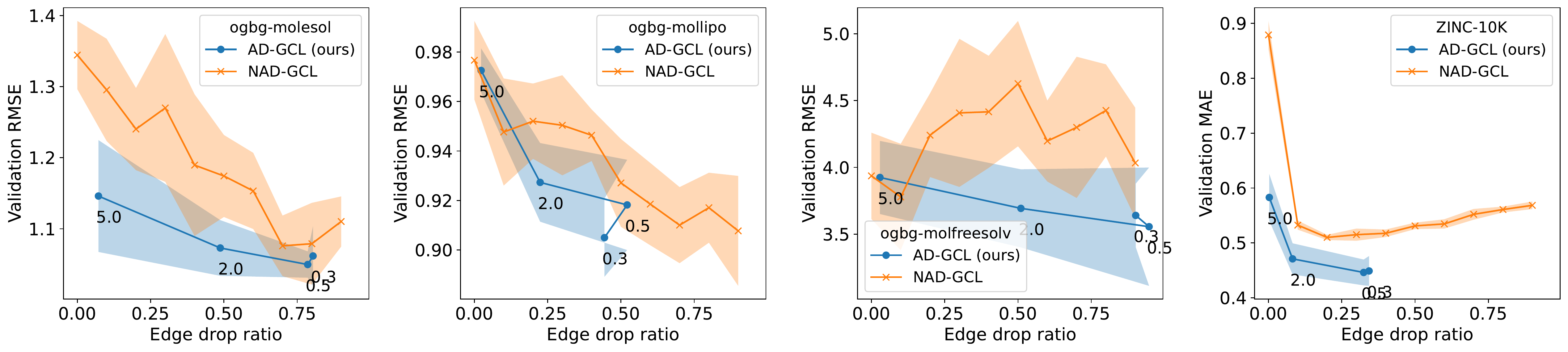}
    \caption{\small{Validation performance for graph regression \textit{v.s.} edge drop ratio. Comparing AD-GCL and GCL with non-adversarial random edge dropping. The markers on AD-GCL’s performance curves show the $\lambda_{\text{reg}}$ used. Note here that lower validation metric is better.} }
    \label{fig:appendix_perf_vs_edge_drop_regression}
\end{figure}

\begin{figure}
    \centering
    \includegraphics[width=0.85\textwidth]{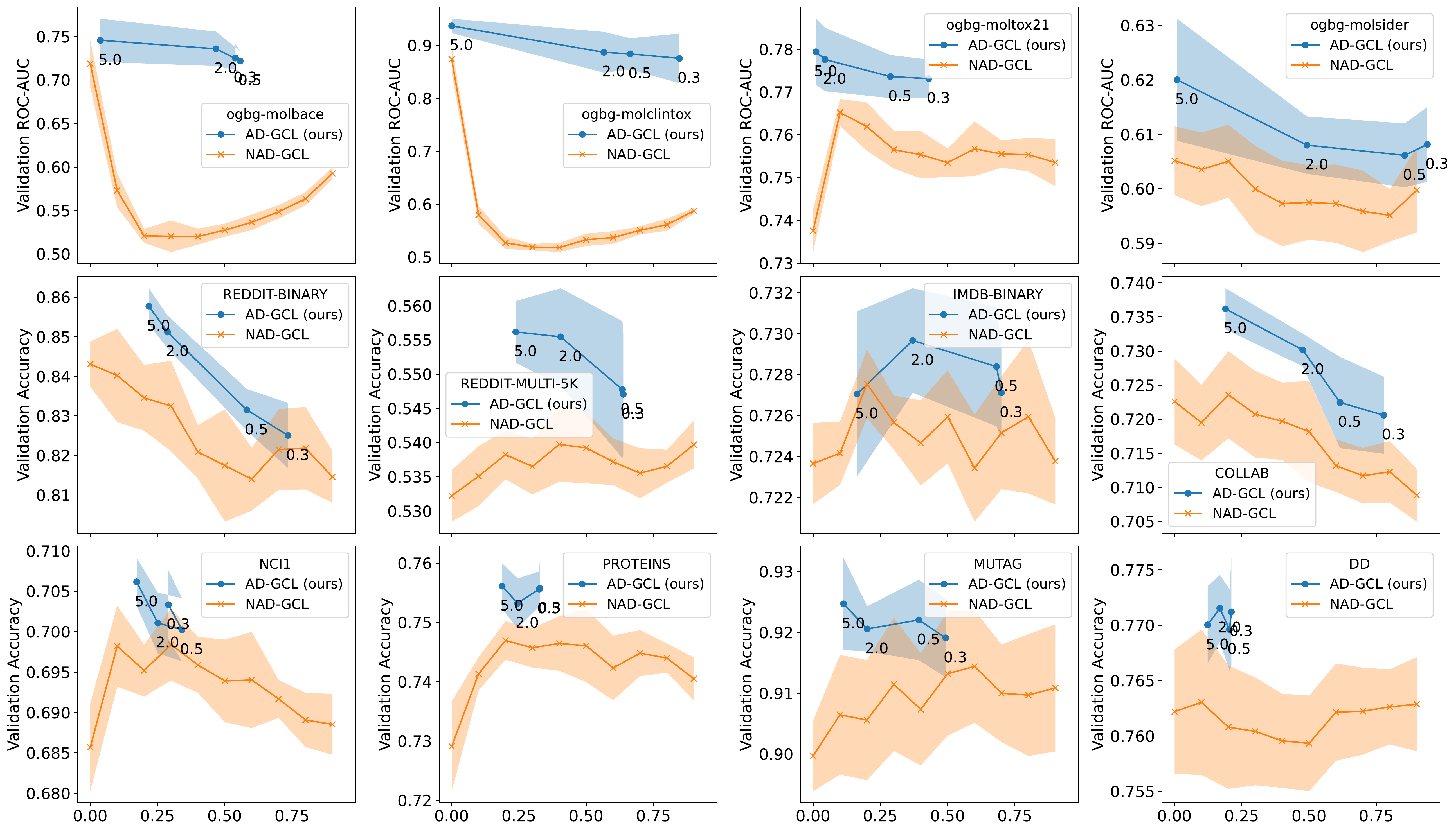}
    \caption{\small{Validation performance for graph classification \textit{v.s.} edge drop ratio. Comparing AD-GCL and GCL with non-adversarial random edge dropping. The markers on AD-GCL’s performance curves show the $\lambda_{\text{reg}}$ used. Note here that higher validation metric is better.}}
    \label{fig:appendix_perf_vs_edge_drop_classification}
\end{figure}

\begin{figure}
    \centering
    \includegraphics[width=0.8\textwidth]{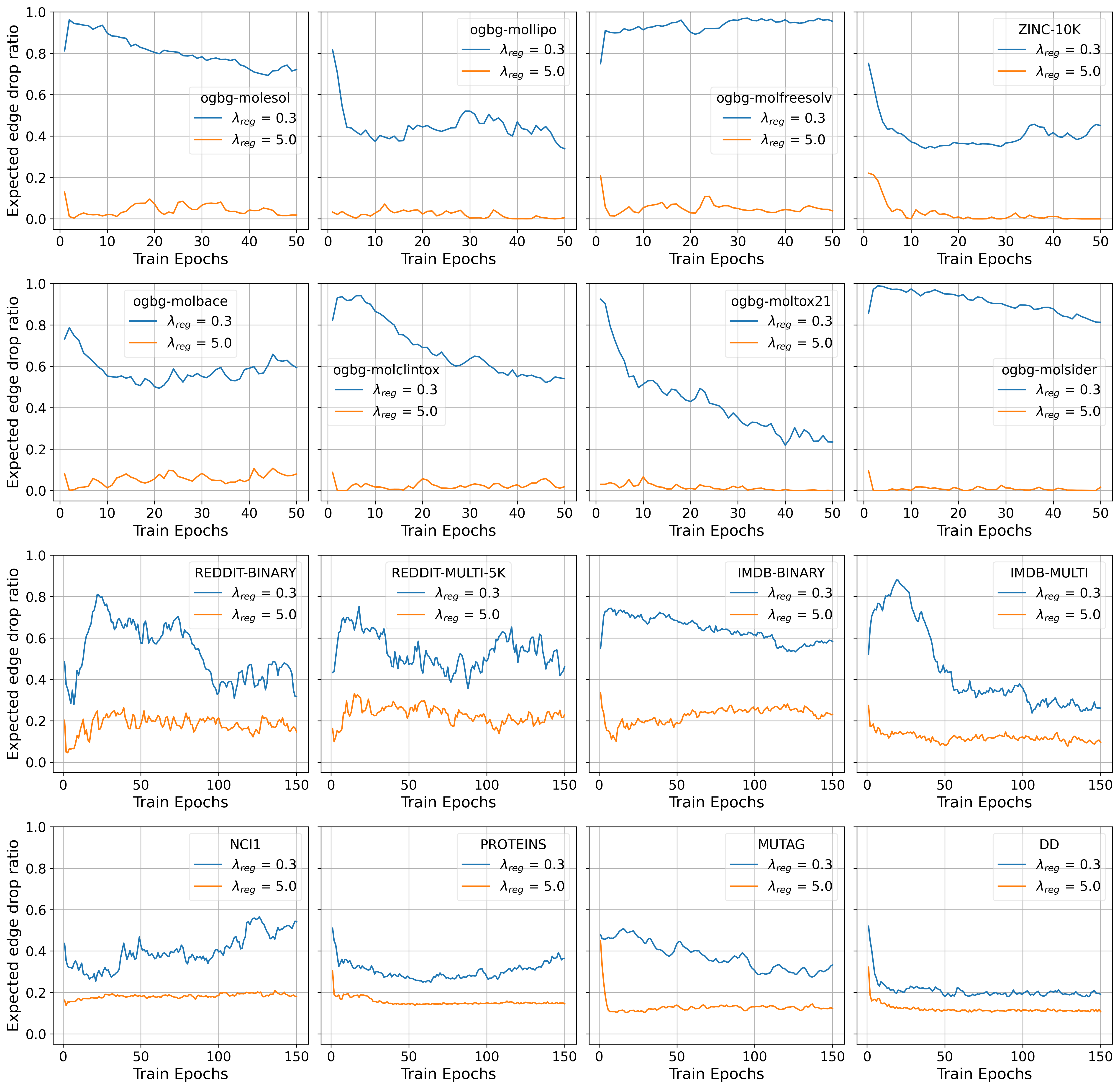}
    \caption{Training dynamics of expected edge drop ratio for $\lambda_{\text{reg}}$.}
    \label{fig:appendix_exp_edge_drop_vs_epochs}
\end{figure}
Figure~\ref{fig:appendix_lambda_vs_exp_edge_drop} shows how different regularization strengths ($\lambda_{\text{reg}}$) affects the expected edge drop ratio for multiple datasets. These results further provides us evidence that indeed, $\lambda_{\text{reg}}$ and the expected edge drop ratio are inversely related in accordance with our design objective and thus provides us with a way of controlling the space of augmentations for our learnable edge dropping GDA. 

Figure~\ref{fig:appendix_perf_vs_edge_drop_classification} shows the complete validation set performance for different edge drop ratios. AD-GCL is compared to a non-adversarial random edge dropping GCL (NAD-GCL). We choose $\lambda_{\text{reg}}$'s that result in an expected edge drop ratio (measured at saddle point of Eq.~\ref{eq:ad-gcl-encoder-augmentor-reg}) value to match the random drop ratio used for NAD-GCL.

Figure~\ref{fig:appendix_exp_edge_drop_vs_epochs} further provides additional plots of the training dynamics of expected edge drop ratio for different $\lambda_{\text{reg}}$ values.

\newpage
\section{Experimental Settings and Complete Evaluation Results}
\label{apd:exp_settings}
In this section, we provide the detailed experimental settings and additional experimental evaluation results for unsupervised, transfer and semi-supervised learning experiments we conducted (Section~\ref{sec:exp}). In addition we also provide details of the motivating experiment (Figure~\ref{fig:bad-exm} in main text). 

\subsection{Motivating Experiment (Figure~\ref{fig:bad-exm})}
\label{apd:exp_settings_motivation}
The aim of this experiment is to show that having GNNs that can maximize mutual information between the input graph and its representation is insufficient to guarantee their performance in the downstream tasks, because redundant information may still maximize mutual information but may degenerate the performance. To show this phenomenon, we perform two case studies: (1) a GNN is trained following the vanilla GCL (InfoMax) objective and (2) a GNN is trained following the vanilla GCL (InfoMax) objective while simultaneously a linear classifier that tasks the graph representations output by the GNN encoder is trained with random labels. These two GNNs have exactly the same architectures, hyperparametes and initialization. Specifically, the GNN architecture is GIN~\cite{xu2018how}, with embedding dimension of 32, 5 layers with no skip connections and a dropout of 0.0. 

Both GNN encoders are trained as above. In the first step of the evaluation, we want to test whether these GNNs keep mutual information maximization. For all graphs in the ogbg-molbace dataset, either one of the GNN provides a set of graph representations. For each GNN, we compare all its output graph representations. We find that, the output representations of every two graphs have difference that is greater than a digit accuracy. This implies that either one of the GNN keeps an one-to-one correspondance between the graphs in the dataset and their representations, which guarantees mutual information maximization.   

We further compare these two GNNs encoders in the downstream task by using true labels. We impose two linear classifiers on the output representations of the above two GNN encoders to predict the true labels. The two linear classifiers have exactly the same architecture, hyperparametes and initialization. Specifically, a simple logistic classifier implemented using sklearn~\cite{scikit-learn} is used with L2 regularization. The L2 strength is tuned using validation set. For the dataset ogbg-molbace, we follow the default train/val/test splits that are given by the original authors of OGB~\cite{hu2020open}. Note that, during the evaluation stage, the GNN encoders are fixed while the linear classifiers get trained. The evaluation performance is the curves as illustrated in Figure~\ref{fig:bad-exm}.

\subsection{Unsupervised Learning}

\paragraph{Evaluation protocol.} In this setting, all methods are first trained with the corresponding self-supervised objective and then evaluated with a linear classifier/regressor. We follow~\cite{chen2020simple} and adopt a linear evaluation protocol. Specifically, once the encoder provides representations, a Ridge regressor (+ L2) and Logistic (+ L2) classifier is trained on top and evaluated for regression and classification tasks respectively. Both methods are implemented using sklearn~\cite{scikit-learn} and uses LBFGS~\cite{zhu1997algorithm} or LibLinear~\cite{fan2008liblinear} solvers . Finally, the lone hyper-parameter of the downstream linear model i.e. L2 regularization strength is grid searched among $\{0.001, 0.01,0.1,1,10,100,1000\}$ on the validation set for every single representation evaluation.

For the Open Graph Benchmark Datasets (ogbg-mol*), we directly use the processed data in Pytorch Geometric format which is available online \footnote{\url{https://ogb.stanford.edu/docs/graphprop/}}. The processed data includes train/val/test that follow a scaffolding split. More details are present in the OGB paper~\cite{hu2020open}. Additionally, we make use of the evaluators written by authors for standardizing the evaluation. The evaluation metric varies depending on the task at hand. For regression tasks it is RMSE (root mean square error) and for classification it is ROC-AUC (\%).

For the ZINC-10K dataset~\cite{dwivedi2020benchmarking}, we use the processed data in Pytorch Geometric format that is made available online\footnote{\url{https://github.com/graphdeeplearning/benchmarking-gnns/tree/master/data}} by the authors. We use the same train/val/test splits that are provided. We follow the authors and adopt MAE (mean absolute error) as the test metric.

For the TU Datasets~\cite{Morris+2020}, we obtain the data from Pytorch Geometric Library~\footnote{\url{https://pytorch-geometric.readthedocs.io/en/latest/modules/datasets.html}} and follow the conventional 10-Fold evaluation. Following standard protocol, we adopt Accuracy (\%) as the test metric.

All our experiments are performed 10 times with different random seeds and we report mean and standard deviation of the corresponding test metric for each dataset.

\paragraph{Other hyper-parameters.} The encoder used for ours and baselines is GIN~\cite{xu2018how}. The encoder is fixed and not tuned while performing self-supervised learning  (i.e. embedding dimension, number of layers, pooling type) for all the methods to keep the comparison fair. The reasoning is that any performance difference we witness should only be attributed to the self-supervised objective and not to the encoder design. Details of encoder for specific datasets.
\begin{itemize}
    \item OBG - \textit{emb dim = 300, num gnn layers = 5, pooling = add, skip connections = None, dropout = 0.5, batch size = 32}
    \item ZINC-10K - \textit{emb dim = 100, num gnn layers = 5, pooling = add, skip connections = None, dropout = 0.5, batch size = 64}
    \item TU Datasets - \textit{emb dim = 32, num gnn layers = 5, pooling = add, skip connections = None, dropout = 0.5, batch size = 32}
\end{itemize} 

The optimization of AD-GCL is performed using Adam and the learning rates for the encoder and the augmenter in AD-GCL are tuned among $\{0.01, 0.005, 0.001\}$. We find that asymmetric learning rates for encoder and augmenter tend to make the training non-stable. Thus, for stability we adopt a learning rate of 0.001 for all the datasets and experiments. The number of training epochs are chosen among $\{20, 50, 80, 100, 150\}$ using the validation set.

\subsubsection{Unsupervised learning with non linear downstream classifier}
\label{apd:unsup_nonlinear}
\begin{table}[htb]
\centering
\resizebox{\textwidth}{!}{%
\begin{tabular}{@{}lcccccccc@{}}
\toprule
           & NCI1             & PROTEINS         & DD               & MUTAG            & COLLAB           & RDT-B            & RDT-M5K          & IMDB-B           \\ \midrule
RU-GIN     & 65.40 $\pm$ 0.17 & 72.73 $\pm$ 0.51 & 75.67 $\pm$ 0.29 & 87.39 $\pm$ 1.09 & 65.29 $\pm$ 0.16 & 76.86 $\pm$ 0.25 & 48.48 $\pm$ 0.28 & 69.37 $\pm$ 0.37 \\
InfoGraph  & 76.20 $\pm$ 1.06 & 74.44 $\pm$ 0.31 & 72.85 $\pm$ 1.78 & 89.01 $\pm$ 1.13 & 70.65 $\pm$ 1.13 & 82.50 $\pm$ 1.42 & 53.46 $\pm$ 1.03 & 73.03 $\pm$ 0.87 \\
GraphCL    & 77.87 $\pm$ 0.41 & 74.39 $\pm$ 0.45 & 78.62 $\pm$ 0.40 & 86.80 $\pm$ 1.34 & 71.36 $\pm$ 1.15 & 89.53 $\pm$ 0.84 & 55.99 $\pm$ 0.28 & 71.14 $\pm$ 0.44 \\ \midrule
AD-GCL-FIX & 75.77 $\pm$ 0.50 & \bf{75.04 $\pm$ 0.48} & 75.38 $\pm$ 0.41 &  88.62 $\pm$ 1.27 & \bf{74.79 $\pm$ 0.41}$^\star$ & \bf{92.06 $\pm$ 0.42}$^\star$ & \bf{56.24 $\pm$ 0.39} & 71.49 $\pm$ 0.98 \\
AD-GCL-OPT & 75.86 $\pm$ 0.62 & \bf{75.04 $\pm$ 0.48} & 75.73 $\pm$ 0.51 & 88.62 $\pm$ 1.27 & \bf{74.89 $\pm$ 0.90}$^\star$ & \bf{92.35 $\pm$ 0.42}$^\star$ & \bf{56.24 $\pm$ 0.39} & 71.49 $\pm$ 0.98 \\ \bottomrule
\end{tabular}%
}
\caption{Unsupervised Learning results on TU Datasets using a non-linear SVM classifier as done in GraphCL~\cite{you2020graph}. Averaged Accuracy (\%) $\pm$ std. over 10 runs. This is different from the linear classifier used to show results in Tables~\ref{tab:unsupervised_learning_ogbg} (TOP) and (BOTTOM).}
\label{tab:unsup_tu_non_linear_eval_results}
\end{table}
In our evaluation, we also observe several further benefits of using a downstream linear model in practice, would like to list them here. First, linear classifiers are much faster to train/converge in practice, especially for the large-scaled datasets or large embedding dimensions, which is good for practical usage. We observe that non-linear SVM classifiers induce a rather slow convergence, when applying to those several OGB datasets where the embedding dimensions are 300 (Table~\ref{tab:unsupervised_learning_ogbg} bottom). Second, compared to the linear model, the non-liner SVM may introduce additional hyper-parameters which not only need further effort to be tuned but also weaken the effect of the self-training procedure on the downstream performance.

\subsection{Transfer Learning}
\paragraph{Evaluation protocol.} We follow the same evaluation protocol as done in~\cite{hu2019strategies}. In this setting, self-supervised methods are trained on the pre-train dataset and later used to be test regarding transferability. In the testing procedure, the models are fine-tuned on multiple datasets and evaluated by the labels of these datasets. We adopt the GIN encoder used in \cite{hu2019strategies} with the same settings for fair comparison. All reported values for baseline methods are taken directly from ~\cite{hu2019strategies} and ~\cite{you2020graph}. For the fine-tuning, the encoder has an additional linear graph prediction layer on top which is used to map the representations to the task labels. This is trained end-to-end using gradient descent (Adam).  
\paragraph{Hyper-parameters.} Due to the large pre-train dataset size and multiple fine-tune datasets finding optimal $\lambda_{\text{reg}}$ for each of them can become time consuming. Instead we use a fixed $\lambda_{\text{reg}} = 5.0$ as it provides reasonable performance.
The learning rate is also fixed to 0.001 and is symmetric for both the encoder and augmenter during self-supervision on the pre-train dataset. The number of training epochs for pre-training is chosen among $\{20, 50, 80, 100\}$ based on the validation performance on the fine-tune datasets. The same learning setting for fine-tuning is used by following ~\cite{you2020graph}.

\subsection{Semi-supervised Learning}
\paragraph{Evaluation protocol.} We follow the protocol as mentioned in ~\cite{you2020graph}. In this setting, the self-supervised methods are pre-trained and later fine-tuned with 10\% true label supervision on the same dataset. The representations generated by the methods are finally evaluated using 10-fold evaluation. All reported values for baseline methods are taken directly from~\cite{you2020graph}. For fine-tuning, the encoder has an additional linear graph prediction layer on top which is used to map the representations to the task labels. This is trained end-to-end by using gradient descent (Adam).
\paragraph{Hyper-parameters.}
For the pre-training our model, a fixed $\lambda_{\text{reg}} = 5.0$ and learning rate of 0.001 for both encoder and augmenter is used. The epochs are selected among $\{20, 50, 80, 100\}$ and finally for fine-tuning with 10\% label supervision, default parameters from ~\cite{you2020graph} are used.

\section{Comparison of AD-GCL and JOAO}
\label{apd:joao_compare}
We first clarify the different mechanisms that JOAO~\cite{you2021graph} and AD-GCL adopt. JOAO selects augmentation families from a pool  $\mathcal{A}$ = {NodeDrop, Subgraph,EdgePert, AttrMask,Identical} and defines a uniform prior on them for their inner optimization over all possible augmentation family pairs. (See Section 3.2 and See Eq. 7,8 in \cite{you2021graph}). An important distinction is that JOAO still adopts uniformly random augmentations and the inner optimization only searches over different pairs of uniform augmentations. Whereas, AD-GCL adopts non-uniformly random augmentations, which essentially corresponds to a much larger search space.

Complexity wise, JOAO is more expensive than AD-GCL as, they utilize projected gradient descent to fully optimize the inner optimization step over all possible augmentations $\mathcal{A}$. This is a factor k more expensive than AD-GCL. The factor k in JOAO is currently $|\mathcal{A}|^2 = 4^2 = 16$. This makes it slow to train while still having a restricted search space compared to AD-GCL which on the other hand is both faster and looks at a larger search space for a given augmentation family. In our experiments on a single GPU, JOAO took ~3.2 hrs for training on COLLAB whereas AD-GCL only took ~14.4 mins (0.24 hrs).

Moreover, we derive the min-max principle in a more principled way by illustrating its connection to graph information bottleneck (Theorem~\ref{thm:main}), which explains the fundamental reason and benefits of optimizing graph augmentation strategies.

\subsection{Experimental Comparison}
We provide comparison between JOAO and AD-GCL in unsupervised learning setting with the standard non-linear downstream classifier setting in Table~\ref{tab:joao_unsup_non_linear}. This is done following \cite{you2021graph} for fair comparison.
Now, we provide the comparison between JOAO and AD-GCL using a linear evaluation protocol for unsupervised setting in Table~\ref{tab:joao_unsup_linear}. Specifically, a linear SVM head is used for evaluating the representations learned by the 2 methods for the downstream task. The regularization hyper-parameter of the linear svm is grid-searched among {0.001, 0.01,0.1,1,10,100,1000}. We re-run the code provided by authors of JOAO (available at \url{https://github.com/Shen-Lab/GraphCL_Automated}) with the default parameters for 5 times each with different seeds. The only change done is to the embedding evaluation code to include linear svm as the final prediction head. For all the TU datasets used here, standard 10-Fold evaluation is used to report classification accuracy (\%).

The results in the above table further indicate that AD-GCL performs better than JOAO in 6 of the 8 TU benchmark datasets. The gap in performance is even more clear compared to the non-linear evaluation setting as shown previously in Table~\ref{tab:joao_unsup_non_linear}. Again, we reiterate that the improved performance gains are due to AD-GCL's search of non-uniformly random augmentations.
\begin{table}[]
\centering
\resizebox{\textwidth}{!}{%
\begin{tabular}{@{}lllllllll@{}}
\toprule
Dataset & NCI1       & PROTEINS   & DD         & MUTAG      & COLLAB     & RDT-B      & RDT-M5K    & IMDB-B     \\ \midrule
JOAO    & 78.07±0.47 & 74.55±0.41 & 77.32±0.54 & 87.35±1.02 & 69.50±0.36 & 85.29±1.35 & 55.74±0.63 & 70.21±3.08 \\
JOAOv2  & 78.36±0.53 & 74.07±1.10 & 77.40±1.15 & 87.67±0.79 & 69.33±0.34 & 86.42±1.45 & 56.03±0.27 & 70.83±0.25 \\
AD-GCL-FIX & 75.77±0.50 & 75.04±0.48 & 75.38±0.41 & 88.62±1.27 & 74.79±0.41 & 92.06±0.42 & 56.24±0.39 & 71.49±0.98 \\ \bottomrule
\end{tabular}%
}
\caption{Unsupervised learning showing Averaged Accuracy (\%) ± std. with a non linear SVM downstream classifier and same standard setup as used in \cite{you2021graph}. The results for JOAO and JOAOv2 are taken from \cite{you2021graph}.}
\label{tab:joao_unsup_non_linear}
\end{table}

\begin{table}[]
\centering
\resizebox{\textwidth}{!}{%
\begin{tabular}{@{}lllllllll@{}}
\toprule
Dataset                & NCI1       & PROTEINS   & DD         & MUTAG      & COLLAB     & RDT-B      & RDT-M5K    & IMDB-B     \\ \midrule
JOAOv2 (FIX-gamma=0.1) & 72.99±0.75 & 71.25±0.85 & 66.91±1.75 & 85.20±1.64 & 70.40±2.21 & 78.35±1.38 & 45.57±2.86 & 71.60±0.86 \\
AD-GCL-FIX             & 69.67±0.51 & 73.59±0.65 & 74.49±0.52 & 89.25±1.45 & 73.32±0.61 & 85.52±0.79 & 53.00±0.82 & 71.57±1.01 \\ \bottomrule
\end{tabular}%
}
\caption{Unsupervised learning showing Averaged Accuracy (\%) ± std. with a linear downstream classifier. JOAOv2 results using linear evaluation is obtained by us using code provided by the authors.}
\label{tab:joao_unsup_linear}
\end{table}

\begin{table}[]
\centering
\resizebox{\textwidth}{!}{%
\begin{tabular}{@{}llllllllll@{}}
\toprule
Fine-Tune Dataset & BBBP       & Tox21      & SIDER      & ClinTox    & BACE       & HIV        & MUV        & ToxCast    & PPI        \\ \midrule
JOAO              & 70.22±0.98 & 74.98±0.29 & 59.97±0.79 & 81.32±2.49 & 77.34±0.48 & 76.73±1.23 & 71.66±1.43 & 62.94±0.48 & 64.43±1.38 \\
JOAOv2            & 71.39±0.92 & 74.27±0.62 & 60.49±0.74 & 80.97±1.64 & 75.49±1.27 & 77.51±1.17 & 73.67±1.00 & 63.16±0.45 & 63.94±1.59 \\
AD-GCL-FIX & 70.01±1.07 & 76.54±0.82 & 63.28±0.79 & 79.78±3.52 & 78.51±0.80 & 78.28±0.97 & 72.30±1.61 & 63.07±0.72 & 68.83±1.26 \\ \bottomrule
\end{tabular}%
}
\caption{Transfer learning results showing mean ROC-AUC $\pm$ std. Pre-Training done using ZINC 2M (used for first 8 fine-tune datasets) and PPI-306K (for the last PPI fine-tune dataset). The results for JOAO and JOAOv2 are taken from \cite{you2021graph}. The experimental setting follows \cite{you2021graph}. }
\label{tab:joao_transfer}
\end{table}

\begin{table}[]
\centering
\resizebox{0.8\textwidth}{!}{%
\begin{tabular}{@{}lllllll@{}}
\toprule
Dataset    & NCI1       & PROTEINS   & DD         & COLLAB     & RDT-B      & RDT-M5K    \\ \midrule
JOAO       & 74.48±0.27 & 72.13±0.92 & 75.69±0.67 & 75.30±0.32 & 88.14±0.25 & 52.83±0.54 \\
JOAOv2     & 74.86±0.39 & 73.31±0.48 & 75.81±0.73 & 75.53±0.18 & 88.79±0.65 & 52.71±0.28 \\
AD-GCL-FIX & 75.18±0.31 & 73.96±0.47 & 77.91±0.73 & 75.82±0.26 & 90.10±0.15 & 53.49±0.28 \\ \bottomrule
\end{tabular}%
}
\caption{Semi-supervised Learning with 10\% label rate showing 10-Fold Accuracy (\%). The results for JOAO and JOAOv2 are taken from \cite{you2021graph}. The experimental setting follows \cite{you2021graph}.}
\label{tab:joao_semi_sup}
\end{table}

More comparison on transfer learning and semi-supervised learning is put in Table~\ref{tab:joao_transfer} and Table~\ref{tab:joao_transfer} respectively, where the experimental settings follow Sec.~\ref{sec:exp}. For transfer learning, AD-GCL outperforms JOAO in 7 among 9 datasets, JOAOv2 in 5 among 9 datasets. For semi-supervised learning, AD-GCL outperforms both of them in all 6 datasets.

\section{Limitations and Broader Impact}
\label{apd:broader_impact}
We stress on the fact that self-supervised methods come with a fundamental set of limitations as they don't have access to the downstream task information. Specifically for contrastive learning, the design of contrastive pairs (done through augmentations) plays a major role as it guides the encooder to selectively capture certain invariances with the hope that it can be beneficial to downstream tasks. Biases could creep in during the design of such augmentations that can be detrimental to the downstream tasks and risk learning of sub-optimal or non-robust representations of input data. Our work helps to alleviate some of the issues of hand designed augmentation techniques and provides a novel principle that can aid in the design of learnable augmentations. It also motivates further research into the understanding the inherent biases of family of augmentations and how they affect the downstream tasks. Finally, self-supervised graph representation learning has a lot of implications in terms of either fairness, robustness or privacy for the various fields that have been increasing adopting these methods. 

\section{Compute Resources}
\label{apd:compute_resource}
All our experiments are performed on a compute cluster managed by Slurm Workload Manager. Each node has access to a mix of multiple Nvidia GeForce GTX 1080 Ti (12GB), GeForce GTX TITAN X (12GB) and TITAN Xp (12GB) GPU cards.

\end{document}